\documentclass[twoside,11pt]{article}

\usepackage[preprint]{jmlr2e}

\usepackage[utf8]{inputenc} 
\usepackage[T1]{fontenc}    
\usepackage{lmodern}
\usepackage{hyperref}       
\usepackage{url}            
\usepackage{booktabs}       
\usepackage{amsfonts}       
\usepackage{nicefrac}       
\usepackage{microtype}      
\usepackage{amsmath}
\usepackage{amssymb}
\usepackage{todonotes}
\usepackage{standalone}
\usepackage{makecell}
\usepackage{multirow}
\usepackage{dsfont} 

\usepackage[final]{changes}
\setdeletedmarkup{\textcolor{red}{\sout{#1}}} 

\newtheorem{lemm}{Lemma}
\newtheorem{prop}{Proposition}
\newtheorem{defi}{Definition}
\newtheorem{exam}{Example}

\newcommand{\R}{\mathbb{R}}
\newcommand{\N}{\mathbb{N}}

\newcommand{\E}{\mathbb{E}}
\newcommand{\Var}{\mathbb{V}ar}
\newcommand{\Nor}{\mathcal{N}}
\newcommand{\bSigma}{\boldsymbol{\Sigma}}
\newcommand{\bhSigma}{\boldsymbol{\hat{\Sigma}}}
\newcommand{\bsigma}{\boldsymbol{\sigma}}
\newcommand{\bmu}{\boldsymbol{\mu}}

\newcommand{\bvartheta}{\boldsymbol{\vartheta}}

\newcommand{\bXi}{\boldsymbol{\Xi}}

\newcommand{\bS}{\boldsymbol{S}}
\newcommand{\bzero}{\boldsymbol{\vec{0}}}
\newcommand{\bO}{\boldsymbol{0}}
\newcommand{\bI}{\boldsymbol{I}}
\newcommand{\ELBO}{\mathcal{L}}

\newcommand{\hELBO}{\widehat{\mathcal{L}}}
\newcommand{\KL}{D_{KL}}

\newcommand{\ix}{x^{(i)}}

\newcommand{\iy}{y^{(i)}}

\newcommand{\bmuzi}{\bmu_z^{(i)}}

\newcommand{\bhmuzi}{\hat{\bmu}_z^{(i)}}

\newcommand{\bSigmazi}{\bSigma_z^{(i)}}
\newcommand{\bhSigmazi}{\bhSigma_z^{(i)}}

\newcommand{\bSzi}{\bS_z^{(i)}}
\newcommand{\tr}{\textnormal{tr}}
\newcommand{\diag}{\textnormal{diag}}

\newcommand{\bm}{\boldsymbol{m}}

\newcommand{\Jv}{J_\vartheta}
\newcommand{\neyzi}{z_0^{(i)}}

\title{A Generalised Linear Model Framework for $\beta$-Variational Autoencoders based on Exponential Dispersion Families}

\author{%
  \name Robert Sicks\;\thanks{Corresponding author} \email robert.sicks@itwm.fraunhofer.de\\
  \addr Department of Financial Mathematics\\
  Fraunhofer ITWM\\
  Kaiserslautern\\
  \AND
  \name Ralf Korn \\
  \addr Department of Financial Mathematics\\
  TU Kaiserslautern\\
  Kaiserslautern\\
  \AND
  \name Stefanie Schwaar \\
  \addr Department of Financial Mathematics\\
  Fraunhofer ITWM\\
  Kaiserslautern
}

\editor{Shakir Mohamed}

\usepackage{lastpage}
\jmlrheading{22}{2021}{1-\pageref{LastPage}}{1/21; Revised
	5/21}{9/21}{21-0037}{Robert Sicks, Ralf Korn and Stefanie Schwaar}
\ShortHeadings{A GLM Framework for $\beta$-VAE based on EDF}{Sicks, Korn and Schwaar}

\firstpageno{1}

\begin{document}

\maketitle

\begin{abstract}%
  Although variational autoencoders (VAE) are successfully used to obtain meaningful low-dimensional representations for high-dimensional data, the characterization of critical points of the loss function for general observation models is not fully understood.
  We introduce a theoretical framework that is based on a connection between $\beta$-VAE  and generalized linear models (GLM). The equality between the activation function of a $\beta$-VAE and the inverse of the link function of a GLM enables us to provide a systematic generalization of the loss analysis for $\beta$-VAE based on the assumption that the observation model distribution belongs to an exponential dispersion family (EDF).  As a result, we can initialize $\beta$-VAE nets by maximum likelihood estimates (MLE) that enhance the training performance on both synthetic and real world data sets. As a further consequence, we analytically describe the auto-pruning property inherent in the $\beta$-VAE objective and reason for posterior collapse.
\end{abstract}


\section{Introduction}

%

Variational autoencoders (VAE) (\cite{Kingma} and \cite{Rezende2014}) are described by \citet{Goodfellow2016} as an ``excellent manifold learning algorithm'' due to the fact that the model is forced ``to learn a predictable coordinate system that the encoder can capture''. VAE do so by using a regularization term in order to get to low energy regions. 
According to \citet{LeCun2020}, regularization like in the VAE case helps to keep the energy function smooth, which is desirable for the model to learn meaningful dependencies (e.g. to fill blanks). In contrast, maximum likelihood approaches push down the energy surface only at training sample regions. Therefore, their inherent objective is ``to make the data manifold an infinitely deep and infinitely narrow canyon'' (see \citealt{LeCun2020}).

Learning meaningful dependencies is a desirable concept for advancing deep learning. Hence, there exists an interest in understanding and developing VAE. Recent work aims at explaining and overcoming well-known pitfalls of VAE, such as spurious global optima (see \citealt{Dai2019}), posterior collapse (see \citealt{Lucas2019}) or prior posterior mismatch (see \citealt{Dai2019} and \citealt{Ghosh}). In these works, a Gaussian observation model distribution is assumed. 

In this work, we answer the following research question:

\begin{itemize}
	\item[]\textit{Is there a way to generalize the loss analysis of $\beta$-VAE based on the observation model distribution?}
\end{itemize}

For this, we establish a connection between $\beta$-VAE and generalized linear models (GLM) and provide a framework for analysing $\beta$-VAE based on the observation model distribution. By doing so, we generalize works of \citet{Dai2018}, \citet{Lucas2019} and \citet{Sicks2020}. We provide an approximation to the evidence lower bound (ELBO), which is exact in the Gaussian distribution case (see also \citealt{Dai2018} and \citealt{Lucas2019}) and a lower bound for the Bernoulli distribution case (see also \citealt{Sicks2020}). Further, we analyse the maximum likelihood \replaced{estimates}{estimators} (MLE) of this approximation.\deleted{Using the MLE as initialization, we show that the training performance of a VAE net can be enhanced. Furthermore, do we find an analytical description of the auto-pruning property of $\beta$-VAE, a reason for posterior collapse.}
\added{Given the MLE, we}
	\begin{itemize}
		\item \added{propose a MLE based initialization and show that the training performance of a VAE net can be enhanced. }
		\item \added{find an analytical description of the auto-pruning property of $\beta$-VAE, a reason for posterior collapse. } 
		\item \added{analytically calculate a statistic used for predicting the number of inactive units in a yet to be trained VAE net and show its practical applicability. }
	\end{itemize}

As GLM are based on exponential dispersion families (EDF), the analysis is based on the distribution assumption for the observation model. This is favourable as VAE, based on EDF, are applied in various different fields (with different distribution assumptions), as e.g.:
 anomaly detection using Gaussian distribution (see \citealt{Xu2018a}), molecules representation using Bernoulli distribution (see \citealt{Blaschke2018}), image compression using Bernoulli distribution (see \citealt{Duan2019}) or multivariate spatial point processes using Poisson distribution (see \citealt{Yuan2020}).

This work is structured as follows: In Section \ref{sec:Motivation and Related Work}, we give a motivation as well as an overview of related work. In Section \ref{sec: Theoretical Background and Advancements}, we present the theoretical background and our results that are a consequence of connecting VAE and GLM. Afterwards in Section \ref{sec:Simulation results}, we provide simulations validating our theoretical results. In Section \ref{sec:Conclusion}, we summarize our contributions and point out future research directions.

\section{Motivation and Related Work} \label{sec:Motivation and Related Work}

Our main contribution is to interpret the decoder of a $\beta$-VAE as a GLM (see Section \ref{subsec:The EDF and the decoder as GLM}). This will allow us to identify well-known activation functions as the reciprocal of link functions of GLM. Therefore, we are able to provide a systematic generalization of the loss analysis for VAE based on the assumption that the observation model belongs to an EDF \added{(see Section \ref{subsec:Local Behaviour for general EDF observation model})}. 

Given an affine transformation for a part of the decoder, we derive MLE for an approximation to the $\beta$-VAE objective. Even though the decoder architecture is arguably simple, analysing the critical points and the loss landscapes of VAE helps to understand these models better.
\added{For example \mbox{\citet{Lucas2019}} consider this approach for their analysis with a Gaussian observation model. 
	
Given the derived MLE, we derive weight and bias initializations (see Section \ref{subsec:MLE-based initialization} and Appendix \ref{subsubsec:Initialization}), analyse the auto-pruning of $\beta$-VAE (see Section \ref{subsec:MLE and optimal solutions}) and analytically calculate a statistic used for predicting the number of inactive units  (see Section \ref{subsec:MLE and optimal solutions} and Section \ref{subsec:Latent dimension activities}).}
\deleted{Using this approach with a Gaussian observation model, \mbox{\citet{Lucas2019}} analyse posterior collapse. In this paper, for general observation models (belonging to EDF), we provide an analytical description on the auto-pruning property of $\beta$-VAE, which leads to posterior collapse.}

By auto-pruning, we mean that during training the net sets nodes for the latent space inactive (i.e. to zero) and is presumably not able to activate them again due to local minima or saddle points in the loss surface. 
On the one hand, this property of $\beta$-VAE is desirable, as the model focusses on useful representations. On the other hand, it is considered a problem when too many units become inactive before learning a useful representation. 

To weaken the effect of auto-pruning, different approaches are considered in the literature. \citet{Bowman2015}, \citet{KaaeSonderby2016} and \citet{Huang2018} use annealing of the parameter $\beta$ during training. Our results below suggest that the annealing does not influence the final amount of active units, if training is conducted long enough and the loss surface is smooth. \citet{Lucas2019} make the same observation for the Gaussian observation model case.

Other approaches to tackle posterior collapse are to adjust the training objective.
\citet{Kingma2016} propose an alternative objective, in which the gradient is ignored if the KL-Divergence is below a pre-determined threshold. In \citet{Razavi2018} a variational posterior is chosen so that the posterior collapse cannot happen by design. Hence, an implicit threshold is chosen.
To reduce posterior collapse, \citet{Yeung2017} propose to use masking of groups in the latent dimension during training in the fashion of dropout layers.
\citet{He2019} find empirically that the variational approximation lags behind the true model posterior in the initial stages of training. They propose to train the inference network separately to alleviate posterior collapse.
\citet{Burda2015} consider importance weighting and show in their experiments that their proposed method yields less inactive units. For this, they calculate an activity statistic for each net after training. In this work, we provide a closed form for this statistic that can be calculated without training.

As our work focuses on analysing critical points of $\beta$-VAE, in the following, we give an overview of literature for analysing Autoencoders and VAE as well as on GLM used in the context of neural nets.

\replaced{Various}{Varios} authors have analysed optimal points \replaced{of}{or} the loss surface for squared loss (i.e. a Gaussian observation model) autoencoders.
\added{For autoencoders with linearised activations,} \citet{Bourlard1988} show \deleted{for autoencoders with linearised activations }that the optimal solution is given by the solution of a singular value decomposition (SVD). \citet{Baldi1989a} extend these results and analyse the squared error loss of autoencoders for all critical points. 
\citet{Zhou2018} provide analytical forms for critical points and characterize the values of the corresponding loss functions as well as properties of the loss landscape for one-hidden layer ReLU autoencoders.
\citet{Kunin2019} consider regularizations on linear autoencoders and analyse the loss-landscape for different regularizations. They show that regularized linear autoencoders are capable of learning the principal directions and have connections to pPCA. 
 
Variational Autoencoders with Gaussian observation models have been considered in \citet{Dai2018}, \citet{Dai2019}, \citet{Lucas2019} and \citet{Dai2020}.
\citet{Dai2018} analyse Gaussian VAE and show connections to pPCA and robust PCA as well as smoothing effects for local optima of the loss landscape. 
\citet{Dai2019} analyse deep Gaussian VAE. Assuming the existence of an invertible and differentiable mapping between low-rank manifolds in the sample space and the latent space, they show that spurious global optima exist, which do not reflect the data-generating manifold appropriately.
\citet{Lucas2019} extend results of \citet{Dai2018} to analyse posterior collapse. They do this by analysing the difference from the true marginal to the ELBO which is possible under Gaussian assumptions, as $P_{\theta}(Z|X)$ becomes tractable in their setting. Furthermore, they provide experimental results on the posterior collapse for deep non-linear cases.
For Gaussian observation models, \citet{Dai2020} introduce a taxonomy for different types of posterior collapse. Furthermore, they show that apart from the KL-Divergence, bad local minima, inherent in deep autoencoders, can lead to posterior collapse.

In this work, we use a linearisation of the decoder.
\citet{Rolinek2019} analyse $\beta$-VAE and show that local orthogonality is promoted on the decoder.
\citet{Kumar2020a} generalize the work of \citet{Rolinek2019} to different observation models and show that diagonal covariances of the variational distribution naturally encourages orthogonal columns of the Jacobian of the decoder. We extend their work as we provide an alternative formulation as well as critical points and error bounds for their approximation.
\citet{Sicks2020} formulate a lower bound for the ELBO of a Bernoulli observation model using the linearisation of the decoder. They use the MLE to derive an initialization scheme and empirically compare it on synthetic data.

\citet{Wuthrich2020} describes connections between GLM and neural network regression models, by interpreting the last layer of a neural net as GLM. With this, he is able to use a $L^1$ regularized neural net to learn representative features to improve a standard  GLM. Furthermore, favourable properties (as for an actuarial context, the ''balance property``) are achieved by a proposed hybrid model.

\section{Analysing the \texorpdfstring{$\beta$}{TEXT}-VAE objective}\label{sec: Theoretical Background and Advancements}

For realizations $x^{(1)},\ldots, x^{(N)}$ of a random variable (r.v.) $X$, we consider a $\beta$-VAE as in \citet{Higgins2017} with the objective $\mathcal{L}$, given by
\begin{align}
\ELBO(\phi,\theta) :=& \dfrac{1}{N} \sum_{i=1}^{N} \E_{Z \sim q_{\phi}\left(\cdot|\ix\right)}\left[\log P_{\theta}(\ix|Z)\right] -\beta \KL\left(q_{\phi}(Z|\ix)|| P_{\theta}(Z) \right),\label{eq:ELBO_up_here}
\end{align}
where $\beta\geq0$. Interpreting the expectation in this expression as autoencoder yields the encoder $q_{\phi}(Z|\ix)$ and the decoder $P_{\theta}(\ix|Z)$. We make the usual assumptions (see \citealt{Kingma}) $P_{\theta}(Z) \sim \Nor(\bzero, \bI)$ and $P_{\theta}(Z|X)$ is approximated by the recognition model with variational distribution 
\[
q_{\phi}(z|X) \sim \Nor(\bmu_z,\bSigma_z).
\]
We further assume for the encoder parameters $\bmuzi := f_1(\ix,\phi)$ and $0 \prec \bSigmazi := \bSzi{\bSzi}^T$, with $\bSzi := f_2(\ix,\phi)$, where $f_1$ and $f_2$ are arbitrary functions including affine transformations.

In the following, we first provide results for the special case of a Gaussian observation model. Then, we provide the theoretical background on EDF and show how the decoder can be interpreted as GLM. Finally, given this new perspective, we present an approximation to the objective in \eqref{eq:ELBO_up_here}, MLE for this approximation and use these MLE to describe the auto-pruning of $\beta$-VAE.

\subsection{Behaviour of the \texorpdfstring{$\beta$}{TEXT}-VAE objective for a Gaussian observation model}\label{subsec:Behaviour for a Gaussian observation model}

In this introductory section, \added{we motivate the novel derivations for the EDF distribution families (see Section \ref{subsec:Local Behaviour for general EDF observation model}), by recapitulating known results for the Gaussian observation model case. We mainly follow the line of argument given in \mbox{\cite{Kumar2020a}}.} 

\replaced{Consider a}{we have a look at the} Gaussian observation model with independent \replaced{marginals}{model stochasticity} $P_{\theta}(x | z_0) \sim \Nor\left(\bvartheta,\sigma^2 \bI\right)$. The deterministic component\footnote{For notational reasons (see Section \ref{subsec:The EDF and the decoder as GLM}), we use $\bvartheta$ instead of the commonly used $\bmu$. } $\bvartheta: Z \rightarrow \R^d$ of the decoder maps to the location parameter (in this case the mean) of the observation model $\log P_{\theta}(x | z)$.
To derive their ``Gaussian Regularized Autoencoder'' \citet{Kumar2020a} use a second order Taylor series expansion of the decoder $f_x(z) = \log P_{\theta}(x | z)$ in a $z_0 \in \R^\kappa$
\begin{equation}\label{eq:taylor PooleKumar}
f_x(z) \approx \log P_{\theta}(x | z_0) + J_{f_x}(z_0) ( z- z_0) + \dfrac{1}{2}( z- z_0)^T H_{f_x}(z_0)( z- z_0),
\end{equation}
where $J_{f_x} \in \R^{1 \times \kappa}$ denotes the Jacobian and $H_{f_x} \in \R^{\kappa \times \kappa}$ the Hessian of $f_x$ evaluated at $z_0$. \added{The benefit of this approximation is that we can analytically calculate the expectation term \eqref{eq:ELBO_up_here}. Given an analytically solvable KL-Divergence, the target of the ELBO becomes deterministic, which is beneficial for the analysis of $\beta$-VAE. }

\replaced{\mbox{\citet{Kumar2020a}} show that if piecewise linear activations are considered}{To achieve computationally feasibly Hessians, \mbox{\citet{Kumar2020a}} consider linear activations} for the deterministic component\footnote{\citet{Kumar2020a} denote this component as $g$. } $\vartheta$\replaced{, we get}{ and get} 
\begin{equation}\label{eq:pw Lin Hf PooleKumar}
	H_{f_x}(z) = \Jv^T \left(\nabla_\vartheta^2 \log P_{\theta}(x | z)\right)\Jv,
\end{equation}
where $\Jv \in \R^{d \times \kappa}$ is the Jacobian of $\bvartheta$. By considering piecewise linear functions, we allow for the decoder architecture to have an arbitrary amount of layers with prominent activations like the ReLU (see \citealt{nair2010rectified}) and alternations of this. \citet{Kumar2020a} choose $\neyzi=\E_{q_{\phi}\left(\cdot|\ix\right)}\left(Z\right)$ to remove the first order term in \eqref{eq:taylor PooleKumar}. This is a reasonable choice, but for the sake of later results, we will stick with general Taylor expansion points $\neyzi \in \R^\kappa$. Using \eqref{eq:taylor PooleKumar} in \eqref{eq:ELBO_up_here}, yields the deterministic objective

\begin{align}
\ELBO(\phi,\theta) \approx \hELBO(\phi,\theta) := \dfrac{1}{N} \sum_{i=1}^{N} & \log P_{\theta}(x | \neyzi) + J_{f_x}(\neyzi) ( \bmuzi- \neyzi) \label{eq:Tay VAE Objective}\\&+ \dfrac{1}{2}\tr\left(\Jv^T(\neyzi) \left(\nabla_\vartheta^2 \log P_{\theta}(x | \neyzi)\right)\Jv(\neyzi) \bSigmazi\right)\nonumber\\
&-\beta \KL\left(q_{\phi}(Z|\ix)|| P_{\theta}(Z) \right).\nonumber
\end{align}

\citet{Kumar2020a} argue that for the deterministic approximation to be accurate either higher \added{central} moments of the variational distribution or the higher order derivates ($\nabla^n_z  \log P_{\theta}(x | z), n \geq 3$) need to be small. Based on the EDF representation, we give bounds for this approximation error (see Corollary \ref{corr: F error}). For the Gaussian case, the approximation error vanishes, regardless of the choice for $\bvartheta$. Hence, the Taylor expansion point is arbitrary in this case.

Maximizing w.r.t. $\bSigmazi$ and $\bmuzi$ yields
\begin{equation}\label{eq:optSigma_up_here}
\bhSigmazi = \left( I - \dfrac{1}{\beta}H_{f_x}(\neyzi)\right)^{-1}
\end{equation}
and
\begin{equation}\label{eq:optMu_up_here}
\bhmuzi = \dfrac{1}{\beta}\bhSigmazi\left(J_{f_x}(\neyzi)^T - H_{f_x}(\neyzi) \cdot \neyzi \right) .
\end{equation}
For now\footnote{The piecewise linear case is considered in section \ref{subsec:Local Behaviour for general EDF observation model}.} considering the Gaussian case and $\bvartheta(z)= Wz + b$, with $W \in \R^{d \times \kappa}$ and $b \in \R^d$, we get by substituting the expressions \eqref{eq:optSigma_up_here} and \eqref{eq:optMu_up_here} 
\begin{align}
\ELBO(\theta,\hat{\phi}) = \dfrac{-1}{2N}\sum_{i=1}^N\bigg[&  \left(\ix - b \right)^T C^{-1} \left(\ix - b \right)+ \beta \log |C| + d \log(2\pi \sigma^{2(1-\beta)}) \bigg], \label{eq:Gaussian ELBO pPCA}  
\end{align}

where $C := \left(\sigma^2 \bI + \beta^{-1} WW^T\right)$.  The derivation of this objective is the same as in the proof of Proposition \ref{prop:general prop for VAE target}, which can be found in Appendix \ref{app:proof of prop}. 

The objective in \eqref{eq:Gaussian ELBO pPCA} is equivalent to the objective in \eqref{eq:ELBO_up_here} and reveals the connections of the VAE objective with $\beta=1$ to probabilistic PCA. As stated in \citet{Dai2018}, a solution for $W$ and $b$ can be derived analytically as given in \citet{Tipping1999}. 
Solutions for the general EDF observation model can be found in Section \ref{subsec:MLE and optimal solutions}.

In the next section, we introduce the EDF, to which the Gaussian distribution belongs, and GLM. Further, we state assumptions in order to generalize the approach from this section.

\subsection{The EDF and the decoder of a VAE as GLM}\label{subsec:The EDF and the decoder as GLM}

\citet{Nelder1972} introduce GLM, providing a generalization of linear statistical models and thus of well-known statistical tools, such as analysis of variance (ANOVA), deviance statistics and MLE (see \citealt{McCullagh1989}). 
GLM consist of three parts: A random component $X$ with a distribution belonging to the EDF, a systematic component given as affine mapping of features $Z$ used to estimate $\E(X|Z)$, and a link function connecting these two components. The EDF is defined by the structure of the density.

\begin{defi}
	We say the distribution of $X$ given $Z$ belongs to the exponential dispersion family (EDF), if the (conditional) density can be written as
	\begin{equation}\label{eq:log_density_expfam}
	\log  P_{\vartheta,\varphi}(X| Z) = \dfrac{X \cdot \vartheta(Z) - F(\vartheta(Z))}{\varphi} + K(X,\varphi),
	\end{equation}
	where $F$ and $K$ are one-dimensional functions. $F: \R \rightarrow \R$ is called the log-normalizer. It ensures that integration w.r.t. the density in \eqref{eq:log_density_expfam} over the support of $X$ is equal to one. 
	$\vartheta(Z) \in \Theta$ is the location parameter. $\Theta$ is an open, convex space with 
	\[\Theta = \left\{\vartheta \in \R: \int_x \exp \left(\dfrac{x\vartheta }{\varphi} + K(x,\varphi)\right) dx < \infty\right\}.
	\] 
	$\varphi > 0$ is called the dispersion parameter and is independent of $Z$.
\end{defi}

\added{The EDF is a subset of the more general Exponential Family and differs by the fact that we can identify the dispersion parameter $\varphi$. }Several well-known distributions, like the Gaussian, Bernoulli and Poisson distribution belong to this family. See Table \ref{tab:exp_fam_dists}  for the respective representations. 

{\renewcommand{\arraystretch}{2}
	\begin{table}[!htb]
		\caption{An overview of well-known distributions that can be written as EDF distribution. The functions for the representation as exponential family member as well as $\vartheta$ and $\varphi$ in terms of the natural parameters are displayed.}
		\label{tab:exp_fam_dists}
		\centering
		\begin{tabular}{ccccc}
			\toprule
			Dist. of X &  $F(\vartheta)$ & $K(x,\varphi)$ & $\vartheta$ & $\varphi$ \\ 
			\midrule
			\makecell[c]{$Bin(n,p)$,\\with $n$ fixed} & $n \log\left(1+\exp\left(\vartheta\right)\right)$ & $\log\binom{n}{x}$ & $\log\left(\dfrac{p}{1-p}\right)$ & $1$ \\
			\hline
			\makecell[c]{$Bern(p)$\\$=Bin(1,p)$}& $\log\left(1+\exp\left(\vartheta\right)\right)$ & $0$ & $\log\left(\dfrac{p}{1-p}\right)$ & $1$ \\
			\hline
			\makecell[c]{$\Nor(\mu,\sigma^2)$,\\with $\sigma^2$ fixed} &  $\dfrac{\vartheta^2}{2}$                                                                                                                                                                                                                                                                       & $-\dfrac{x^2}{2\varphi} - \dfrac{\log\left(2\pi\varphi\right)}{2}$ & $\mu$ & $\sigma^2$ \\ 
			\hline
			$Pois(\lambda)$  & $\exp(\vartheta)$ & $- \log\left(x!\right)$ & $\log(\lambda)$ & $1$ \\
			\bottomrule
		\end{tabular}
\end{table}} 

The EDF is studied in \citet{BarndorffNielsen2014}, \citet{Jorgensen1986} and \citet{Jorgensen1987a}. For an EDF distribution, the expectation as well as the variance can easily be computed. Further, the log-normalizer $F$ provides explicit forms of the conditional expectation and variance and has further desirable properties, as can be seen in the following Lemma.

\begin{lemm}\label{lemm: EDF e-wert, F}
	Let the distribution of a one-dimensional r.v. $X\sim P_{\vartheta,\varphi}(X| Z)$ given $Z$ belong to the EDF. Then, it holds $\E(X| Z) = F'(\vartheta(Z))$ and $\Var(X| Z)= \dfrac{1}{\varphi}F''(\vartheta(Z))$.
	Furthermore, the log-normalizer function $F$ is convex and possesses all derivatives. 
\end{lemm}
The proof for the unconditional case is performed in Theorem 7.1, Corollary 7.1 and Theorem 8.1 in \citet{BarndorffNielsen2014}. The statement for the conditional case follows analogously.

We interpret the decoder $P_{\theta}(\ix|Z)$ as GLM. Therefore, we assume that the independent identical marginal distributions of $X$ given $Z$ belong to an EDF, where they share the same $\varphi$. With $Z \sim q_{\phi}\left(\cdot|\ix\right)$ from the encoder, the parameters of the decoder $P_{\theta}(\ix| Z)$ are given by $\theta = \left\{\bvartheta,\varphi\right\}$ and we have

\[
\bvartheta(Z) = \left(\vartheta_1(Z),\ldots,\vartheta_d(Z)\right)^T.
\]

In order for the neural net implementation $\bm \circ \bvartheta: \R^{\kappa} \rightarrow \R^d$ of a decoder to be reasonable for the log-likelihood part in \eqref{eq:ELBO_up_here}, the decoder should approximate the expectation of $\ix$ given $Z$. According to Lemma \ref{lemm: EDF e-wert, F}, the last activation $\bm: \bvartheta(Z) \rightarrow \R^d$ has to be $\bm = F'$ (applied element wise) to get
\begin{equation}\label{eq:decoder as GLM}
\bm(\bvartheta(Z)) = F'(\bvartheta(Z)) = \E_{\bvartheta,\varphi}\left(\ix| Z\right).
\end{equation}

We call the choice of $\bm$ in \eqref{eq:decoder as GLM} ``canonical activation''. This name originates from the ``canonical link function''. As mentioned before, for GLM a link function $g$ connecting the systematic component of the model $\bvartheta(z)$ to the random component $\E_{\bvartheta}\left(X| z\right)$ is used. This function is called canonical if $g = (F')^{-1}$. Hence, the canonical activation is the inverse of the canonical link. In practice various different link functions, or in our case activations, are considered. 

We want to emphasize that common neural net implementations depend on the choice of the last activation function to properly map to the natural parameters of the distribution, as the choice of loss function is strongly connected to this: 
\begin{itemize}
	\item If we use a Mean-Squared-Error loss and therefore implicitly\footnote{Furthermore, $\sigma^2$ is implicitly set to 1/2 which can result in unwanted consequences (see Section \ref{subsec:MLE and optimal solutions}).} assume a Gaussian ($\Nor(\bmu,\sigma^2\bI)$) log-likelihood, the last activation has to be linear. In Section \ref{subsec:Behaviour for a Gaussian observation model}, we have implicitly assumed the last activation $\bm$ to be the identity to ensure 
	\[
	\bm(\bvartheta) = id(\bvartheta)= F'(\bvartheta) = \bmu.
	\]
	\item For the Binary Cross-Entropy loss, we implicitly assume a Bernoulli distribution $Bern\left(p\right)$. Hence, the last activation should be the sigmoid function to ensure \[
	\bm(\bvartheta) = \dfrac{1}{1+\exp(-\bvartheta)} = F '(\bvartheta) = p.
	\]
\end{itemize}

Actually, all activations that are equivalent to the choice in \eqref{eq:decoder as GLM} up to a scalar $\rho \in \R\setminus\{0\}$, i.e.
\begin{equation}\label{eq: linearly canonical mapping}
\bm(\bvartheta) = F'(\rho \cdot \bvartheta),
\end{equation}
are legitimate choices. We call such activations ``linearly canonical activation'' and for canonical activations we have $\rho=1$. The following example shows that the tanh activation can be used as a ``linearly canonical activation''.

\begin{exam}[Bernoulli distribution - tanh activation]
	Assume $X \sim Bern\left(p\left(\bvartheta\right)\right)$ and set the activation as
	\[
	\bm(\bvartheta)= 1/2 \cdot \tanh(\bvartheta) + 1/2.
	\]
	As in the example before, $F(\bvartheta) = \log\left(1 + \exp\left(\bvartheta\right)\right)$ and it can be shown that
	\[
	\bm(\bvartheta) = F '(2 \cdot \bvartheta).
	\]
\end{exam}
Our theory presented in this paper applies for any linearly canonical activation. As we consider piecewise linear functions for $\bvartheta$, we can substitute $\hat{\bvartheta} := \rho \cdot \bvartheta$ and calculate
\[
\bm(\bvartheta(Z)) = F'(\hat{\bvartheta}(Z)) = \E_{\hat{\bvartheta},\varphi}\left(\ix| Z\right).
\]
Therefore, for notational ease we will stick with the canonical activations. During our simulations, settings with either sigmoid or tanh activation were indistinguishable.

\subsection{Local Behaviour of the \texorpdfstring{$\beta$}{TEXT}-VAE objective for EDF observation models}\label{subsec:Local Behaviour for general EDF observation model}

In this section, we derive an approximation to the $\beta$-VAE objective in \eqref{eq:ELBO_up_here} similar to the way in Section \ref{subsec:Behaviour for a Gaussian observation model}, but for a more general case \added{by considering distributions from the EDF. 
	
In their work, \mbox{\citet{Kumar2020a}} consider distributions with finite first and second moments, which is even more general. Using the more restrictive class of EDF distributions, we provide an error characterization for the approximation in \eqref{eq:taylor PooleKumar}.} \replaced{Further, given Proposition \ref{prop:general prop for VAE target}, we derive MLE for the affine decoder case (see Section \ref{subsec:MLE and optimal solutions}). Given these, we produce}{We can use the results to derive} weight and bias initializations (see \added{Section \ref{subsec:MLE-based initialization} and} Appendix \ref{subsubsec:Initialization}\added{)} \deleted{and the results in Section \ref{subsec:MLE-based initialization}), to} \added{and} analyse the auto-pruning of $\beta$-VAE (see Section \ref{subsec:MLE and optimal solutions} and \ref{subsec:Latent dimension activities})\replaced{.}{or to monitor the training of a $\beta$-VAE with a tractable reference point (see Section \ref{subsec:MLE-based initialization})}

Apart from the assumptions in Section \ref{sec: Theoretical Background and Advancements}, for the Taylor Series expansion based on the decoder $\bm \circ \bvartheta$ as in Section \ref{subsec:Behaviour for a Gaussian observation model}, we further assume 
\begin{itemize}
	\item $\bvartheta$ to be piecewise linear,
	\item a canonical activation function $\bm$ and 
	\item the Taylor expansion points $\neyzi$ from \eqref{eq:taylor PooleKumar} belong to the null space (kernel) of $\bvartheta$: $\neyzi \in ker(\bvartheta)$.
\end{itemize} 

\begin{prop}\label{prop:general prop for VAE target}
	Assume that the independent identical marginals of $X$ given $Z$ belong to the same EDF distribution with functions $F$ and $K$ as in \eqref{eq:log_density_expfam}. Under the assumptions stated in the beginning of this section, there exists an approximative representation for the VAE objective in \eqref{eq:ELBO_up_here},
	\begin{equation}\label{eq: elbo approx helbo}
	\ELBO(\theta,\phi) \approx \hELBO(\theta,\phi),
	\end{equation} 
	that admits optimal solutions $\hat{\phi} = \left\{\bhmuzi ,\bhSigmazi; i=1,\ldots,N\right\}$ (given in \eqref{eq:optSigma_up_here} and \eqref{eq:optMu_up_here}), such that it can be written as
	\begin{align}
	\hELBO(\theta) := \hELBO(\theta, \hat{\phi}) = 
	\dfrac{-1}{2N} \sum_{i=1}^N \Bigg[&\left(F''(0)^{-1}(\ix-F'(0)) + \Jv(\neyzi)\neyzi \right)^T C(\neyzi)^{-1}\nonumber\\
	&\quad\quad\left(F''(0)^{-1}(\ix-F'(0)) + \Jv(\neyzi)\neyzi \right)\nonumber\\
	&+ \beta\log \left|C(\neyzi)\right|+ \beta \cdot d\log \left(\varphi^{-1} F''(0)\right) +D\left(\varphi\right)\Bigg],\label{eq:prop alt Target}
	\end{align}
	where $C(\neyzi) := F''(0)^{-1} \varphi  \bI_d + \beta^{-1} \Jv(\neyzi)\Jv(\neyzi)^T$ and $\Jv(\neyzi) \in \R^{d\times\kappa}$ is the Jacobian of $\bvartheta$. The definition of $D(\varphi)$ can be found in equation \eqref{eq:definition of D(varphi)} of the appendix. 
\end{prop} 

The proof can be found in Appendix \ref{app:proof of prop}. 

By choosing a common Taylor expansion point $\neyzi = z_0$ for all $i =1,\ldots,N$ (i.e. for all observations), Proposition \ref{prop:general prop for VAE target} shows that the local approximation of the $\beta$-VAE objective for different EDF admits a pPCA fashioned representation. Furthermore, this representation belongs to the matrix perspective functions class in \citet{Won2020}, which can be optimized using proximity operators.

See Table \ref{tab:comb_exp_act} for different EDF distribution associated parameters. Unfortunately, this approximation is not possible for all EDF distributions. As an example the canonical activation of the Gamma distribution (which also belongs to the EDF) is given by $-1/x$, with support in $\R^{-}$. Hence, we cannot choose $\neyzi \in ker(\bvartheta)$.

{\renewcommand{\arraystretch}{2}
	\begin{table}[!htb]
		\caption{EDF distribution associated parameters in Proposition \ref{prop:general prop for VAE target}.}
		\label{tab:comb_exp_act}
		\centering
		\begin{tabular}{ccccc}
			\toprule
			Dist. of $X|Z$ & $F(0)$ & $F'(0)$ & $F''(0)$ & $\beta \cdot d\log \left(\varphi^{-1} F''(0)\right) + D(\varphi)$ \\ 
			\midrule
			Bern($p$)& $\log(2)$ & $1/2$ & $1/4$& $\left(1-\beta\right) d \log(4) - 4 N^{-1} \sum_{i=1}^{N} \left\|\ix - 1/2\right\|_2^2 $ \\
			$\Nor(\mu,\sigma^2)$ & $0$ & $0$ & $1$& $d \log(2\pi \sigma^{2(1-\beta)})$ \\ 
			$Pois(\lambda)$  & $1$ & $1$ & $1$& $2d + N^{-1} \sum_{i=1}^{N}\left[- \left\|\ix - 1\right\|_2^2 + 2 \log\left(\prod_{j=1}^{d}\ix_j!\right)\right]$\\
			\bottomrule
		\end{tabular}
\end{table}}

%
%
%

In the following C\deleted{c}orollary, we quantify the introduced Taylor remainder in \eqref{eq:taylor PooleKumar} for different distributions.

\begin{corollary}\label{corr: F error}
	Let the assumptions of Proposition \ref{prop:general prop for VAE target} be given. We introduce the remainder of a second order Taylor term $T(z;\neyzi)$ in \eqref{eq:taylor PooleKumar}, by
	\[
	f_x(z) - T(z;\neyzi) = R_2(z;\neyzi).
	\]
	\begin{itemize}
		\item For a Gaussian observation model, we have $R_2(z;\neyzi)=0 \quad\forall \neyzi \in Z$ and hence in \eqref{eq: elbo approx helbo}
		\[
		\ELBO(\theta,\phi) = \hELBO(\theta,\phi).
		\]
		\item For a Binomial observation model, we obtain
		\[ 
		R_2(z;\neyzi) = \dfrac{n}{8 \cdot 4!} \left\|\Jv(\xi)(z-\neyzi)\right\|_4^4 \cdot M,
		\]
		with $M \in \left[ \dfrac{-1}{3},1\right]$ and $\xi= \neyzi + c \left(z - \neyzi\right)$, where $c \in [0,1]$.
		Further, if we assume $\bvartheta$ to be affine on the convex set spanned by $z$ and $\neyzi$, we have $M \in \left[0,1\right]$ and hence 
		\begin{equation}\label{eq:bound < L < zero}
		\ELBO(\theta,\phi) \geq \hELBO(\theta,\phi).
		\end{equation}
		\item For a Poisson observation model, if we assume $\bvartheta$ to be affine on the convex set spanned by $z$ and $\neyzi$, it can be shown that we have
		\[
		\sum_{j=1}^d -\bvartheta_j(z)^3 \cdot \exp( \bvartheta_j(z))/6 \leq  R_2(z,\neyzi) \leq \sum_{j=1}^d -\bvartheta_j(z)^3/6.
		\]
	\end{itemize}
\end{corollary}
See Appendix \ref{app:proof of corollary F error} for the proof. 

If we choose $\beta=1$, the objective in \eqref{eq:ELBO_up_here} becomes the ELBO. For $\bvartheta(z)= Wz+b$, with $W \in \R^{d \times \kappa}$ and $b \in \R^d$, Corollary \ref{corr: F error} highlights how our theory generalizes the works of \citet{Dai2018}, \citet{Lucas2019} and \citet{Sicks2020}. Under the Gaussian assumption $\hELBO$ is exact. Then, Proposition \ref{prop:general prop for VAE target} yields the objective in \eqref{eq:Gaussian ELBO pPCA} also given by \citet{Dai2018} and analysed by \citet{Lucas2019}. For the Bernoulli distribution, according to \eqref{eq:bound < L < zero} we approximate the ELBO in \eqref{eq:ELBO_up_here} from below. The result is the same lower bound as reported in \citet{Sicks2020}. Thus, the ELBO is bounded from both sides as naturally its values have to be smaller than zero. \added{Further, as we will see in the simulations, the expected error $\E_{q_{\hat{\phi}}}[R_2(\hat{\theta})]:= \dfrac{1}{N} \sum_{i=1}^N \E_{q_{\hat{\phi}}}[R_2(z^{(i)};\neyzi)]$ evaluated at $M=1$ and $\theta = \hat{\theta}$ serves as an indicator for us on what to expect from training. }

\subsection{MLE for the affine transformation case}\label{subsec:MLE and optimal solutions}

In this section, we derive analytical solutions for the objective in \eqref{eq:prop alt Target}, when the location parameter is given by an affine transformation $\bvartheta(z)=Wz+b$. We analyse the optimal values of $W \in \R^{d \times \kappa}$ and $b \in \R^d$ and highlight interesting implications. 

First, we rewrite the objective in \eqref{eq:prop alt Target}. With $\bvartheta(z)=Wz+b$, we obtain a similar representation as \citet{Tipping1999} for pPCA, given by

\begin{equation}\label{eq:hELBO Tipping Bishop}
\hELBO(W,b) = \dfrac{-1}{2} \left( \tr\left(C^{-1}S\right) + \beta \log|C| + \beta \cdot d\log \left(\varphi^{-1} F''(0)\right) + D\left(\varphi\right)\right),
\end{equation}
where $C := \left( F''(0)^{-1} \varphi I_d + \beta^{-1} WW^T \right)$ and 
\[
S:= \dfrac{1}{N}\sum\limits_{i=1}^{N} \left(F''(0)^{-1}\left(\ix-F'(0)\right) - b \right)\left(F''(0)^{-1}\left(\ix-F'(0)\right) - b \right)^T.
\]
According to \citet{Tipping1999}, the MLE for $\hat{b}$ is given by the sample mean 
\begin{equation}\label{eq:b MLE sample mean}
\hat{b} = \dfrac{1}{N}\sum_{i=1}^N F''(0)^{-1}\left(\ix-F'(0)\right). 
\end{equation}
Therefore, $S$ with $\hat{b}$ becomes the sample covariance, which we denote as $\hat{S}$. With $\lambda_1, \ldots,\lambda_d$ we denote the (ordered) eigenvalues of the matrix $\hat{S}$ and in a similar way to \citet{Tipping1999}, for $F''(0) \leq 1$, we can derive the MLE of $W$ as
\begin{equation}\label{eq:opt_w}
\hat{W}=  U_{\kappa} \left(K_{\kappa}- \beta F''(0)^{-1}\varphi I_{\kappa}\right)^{1/2} R  =: U_{\kappa} L R,
\end{equation}
where $U_{\kappa} \in \R^{d\times\kappa}$ is composed of $\kappa$ eigenvectors of the matrix $\hat{S}$. The eigenvectors are associated with the $\kappa$ biggest eigenvalues $\lambda_1, \ldots,\lambda_\kappa$. $K_{\kappa} \in \R^{\kappa\times\kappa}$ is a diagonal matrix with entries 
\begin{equation}\label{eq: k bigger 1/alpha}
k_j=\left\{\begin{array}{ll} \lambda_j, & \lambda_j \geq \beta F''(0)^{-1}\varphi  \\
\beta F''(0)^{-1}\varphi , & \textnormal{else.}\end{array}\right.
\end{equation}
$R \in \R^{\kappa\times\kappa}$ is an arbitrary rotation matrix, which implies that our optimal solution is invariant to rotations. \citet{Dai2018} show this as well as invariance to permutations in their Theorem 2. 

Further for the Gaussian case, the MLE for $\varphi = \sigma^2$ is given by 
\begin{equation}\label{eq:sigma-estimator}
\hat{\sigma}^2 = \dfrac{1}{(d-\beta\kappa)} \sum_{i=\kappa+1}^{d} \lambda_i,
\end{equation}
which can be interpreted \deleted{this }as the variance lost due to the dimension reduction by the autoencoder. This expression is only well-defined for $\beta \in [0, d/\kappa)$ and we have $\hat{\sigma}^2 > 0$ if $rank(S)>\kappa$. Further, the estimator $\hat{\sigma}^2$ is increasing in $\beta$. Hence the VAE performs optimal in view of reconstruction (has the lowest variance lost), when $\beta=0$. 

\replaced{This observation agrees with the definition of the objective \eqref{eq:ELBO_up_here}: lower $ \beta$ emphasize the reconstruction part. As pointed out by an anonymous reviewer, our observation is in line with the analysis by \mbox{\cite{Alemi2018}} as well as \mbox{\cite{Rezende2018}}. While the results therein are originated from a different point of view, the interpretations on $\beta$ are consistent.}{This observation agrees with the definition of the objective \eqref{eq:ELBO_up_here}: lower $\beta$ emphasize the reconstruction part.}

It is possible to have $\textnormal{rank}(\hat{W})< \kappa$ as the ``cut-off'' term
\begin{equation}
\label{eq:cut-off}
\beta F''(0)^{-1}\varphi 
\end{equation}
controls how much columns in the matrix $ \hat{W}R^T$ are zero. We interpret this as a consequence of the auto-pruning property of VAE. If the data signal is not strong enough, it is pruned away. In common VAE implementations $\sigma^2$ is often implicitly assumed to be equal to $1/2$ (i.e. when using MSE loss without any scaling). \citet{Lucas2019} show how the stability of the estimator for $W$ is influenced by the choice of $\sigma^2$. If $\sigma^2$ (and hence the cut-off value) is chosen too large, principal components cannot be captured by the model. We agree with their conclusion that learning $\sigma^2$ is necessary for gaining a full latent representation.

For the parameter estimates of the variational distribution, we get 
\begin{equation}\label{eq:optSigma with opt W}
\bhSigma_z =  \dfrac{\beta \varphi}{F''(0)}R^T  K_{\kappa}^{-1}R
\end{equation}
and
\begin{equation}\label{eq:optMu with opt W}
\bhmuzi = \dfrac{1}{\beta\varphi}\bhSigma_z \hat{W}^T \left(\ix - \bar{x}\right)= R^T  K_{\kappa}^{-1} L  U_{\kappa}^T \dfrac{1}{F''(0)} \left(\ix - \bar{x}\right). 
\end{equation}
When a diagonal covariance structure is imposed, the decoder Jacobian columns are forced to be orthogonal. In \eqref{eq:optSigma with opt W} a diagonal covariance matrix means $R=I_\kappa$. As a result, we have orthogonal columns in the matrix $\hat{W}$. This result supports the findings of \citet{Kumar2020a}. They show an implicit regularization in the local behaviour of the VAE objective \eqref{eq:ELBO_up_here} for a diagonal covariance assumption, without presenting analytical solutions as the one in \eqref{eq:optSigma with opt W}.

Next, we analyse how the parameter $\beta$ influences the optimal variational parameters and as a consequence the auto-pruning of $\beta$-VAE.
\begin{itemize}
	\item For $\beta$ high enough, we get $\hat{W}=\bO$ and $\bSigma_z = \bI_\kappa$ and hence $\bmuzi = 0$ independent of the input $\ix$. Therefore, the Kullback-Leibler Divergence part in \eqref{eq:ELBO_up_here} is amplified enough such that the variational distribution generates independent noise. The posterior collapses.
	\item For smaller $\beta$ values, more and more eigenvalue dimensions covered by $U_{\kappa}^T$ are used and scaled appropriately with $K_{\kappa}^{-1} L$. Therefore, in $\bmuzi$ the inputs $F''(0)^{-1}\left(\ix - \bar{x}\right)$ are transformed better and better to the latent space to guarantee a proper reconstruction.
\end{itemize}

We can further analytically compute statistics used to detect active latent dimensions in $\beta$-VAE. \citet{Burda2015} propose the statistic $A_{z_j} = Cov_x\left(\E_{q_{\phi}\left(z_j|x\right)}[z_j]\right)$. They define the dimension $z_j$ to be active if $A_{z_j} > 0.01$. Using the sample covariance to approximate this value with the given data points and using \eqref{eq:optMu with opt W}, we get
\begin{equation}\label{eq:activity analytical}
A_{z_j} \approx \dfrac{\left(k_j-\beta F''(0)^{-1}\varphi\right)\lambda_j}{k_j^2} = \left\{\begin{array}{ll} \dfrac{\left(\lambda_j-\beta F''(0)^{-1}\varphi\right)}{\lambda_j}, & \lambda_j \geq \beta F''(0)^{-1}\varphi  \\
0 , & \textnormal{else.}\end{array}\right.
\end{equation}
So the value is either $0$ or equals the (positive) relative distance of the eigenvalue $\lambda_j$ to the cut-off value $\beta F''(0)^{-1}\varphi$. The effect on how $\beta$ controls the activity of the latent space dimensions becomes apparent. The bigger $\beta$ the less latent dimensions remain non-zero. 

This result yields the ineffectiveness of annealing the $\beta$ parameter during training. If training is conducted long enough and the loss surface is smooth enough, the MLE will be achieved by optimization. Hence, the active units are determined by \eqref{eq:activity analytical} for the last $\beta$ value during annealing.

\section{Simulation results}\label{sec:Simulation results}

In this section, we provide simulation results to illustrate our theoretical results from Section \ref{sec: Theoretical Background and Advancements}. 
We consider two applications.
\begin{enumerate}
	\item We show that the use of the MLE derived in section \ref{sec: Theoretical Background and Advancements} as initialization for VAE implementations yields a faster training convergence.
	\item We compare the analytical calculations of the activity statistics in \eqref{eq:activity analytical} with the resulting activities of $\beta$-VAE implementations. The analytical values serve as good indicator on how much latent dimensions become inactive during training.
\end{enumerate}

\subsection{MLE-based initialization}\label{subsec:MLE-based initialization}
We focus on the Bernoulli case, popular for image data and set $\beta=1$. According to Corollary \ref{corr: F error}, $\hELBO = \hELBO(\theta)$ from Proposition \ref{prop:general prop for VAE target} becomes a lower bound yielding \eqref{eq:bound < L < zero}. Therefore, we expect the ELBO of VAE with an according architecture to lie above $\hELBO$. The essential messages of the simulations are the following:
\begin{itemize}
	\item  It is reasonable to use $\hELBO$ to analyse the training performance on real life data sets.
	\item The statement above is also valid for ReLU-Net decoders.
	\item The MLE points, from Section \ref{subsec:MLE and optimal solutions}, used as initialization enhance the training performance. 
\end{itemize}

For training of the nets, we use the Adam optimizer by \citet{Kingma2015} with learning rate $0.0001$ and a batch size of $100$. Training was done for a total of $25,000$ batch evaluations. The simulations ran on a dual Intel Xeon E5-2670  with 16 CPU @ 2.6 GHz. The longest setup took about one hour of computing time.

By varying the following hyper parameters, we conduct a total of 18 different simulation setups:
\begin{itemize}
	\item Architecture: ``Affine'' or ``ReLU-Net'' decoder.
	\item Latent dimension $\kappa$: 2, 5 or 20.
	\item Data: ``synthetic'', ``frey'' or ``mnist''.
\end{itemize}
We compare our initialization scheme (``MLE-B'') to a benchmark (``Bench'') given by \citet{He2015}. The initialization schemes and different hyper parameters are explained in detail in Appendix \ref{app:Simulation}. Figure \ref{fig:data_frey_kappa_2_act_sigmoid_plot} shows the result of training the two different initialized VAE on the frey data set with $\kappa=2$. \added{The curves are based on simulations with 10 different seeds. We display the average training performances with pointwise 0.95 confidence intervals.}

\begin{figure*}[!htb]
	\vskip -0.2in
	\begin{center}
		\centerline{\includegraphics[width=1.2\linewidth]{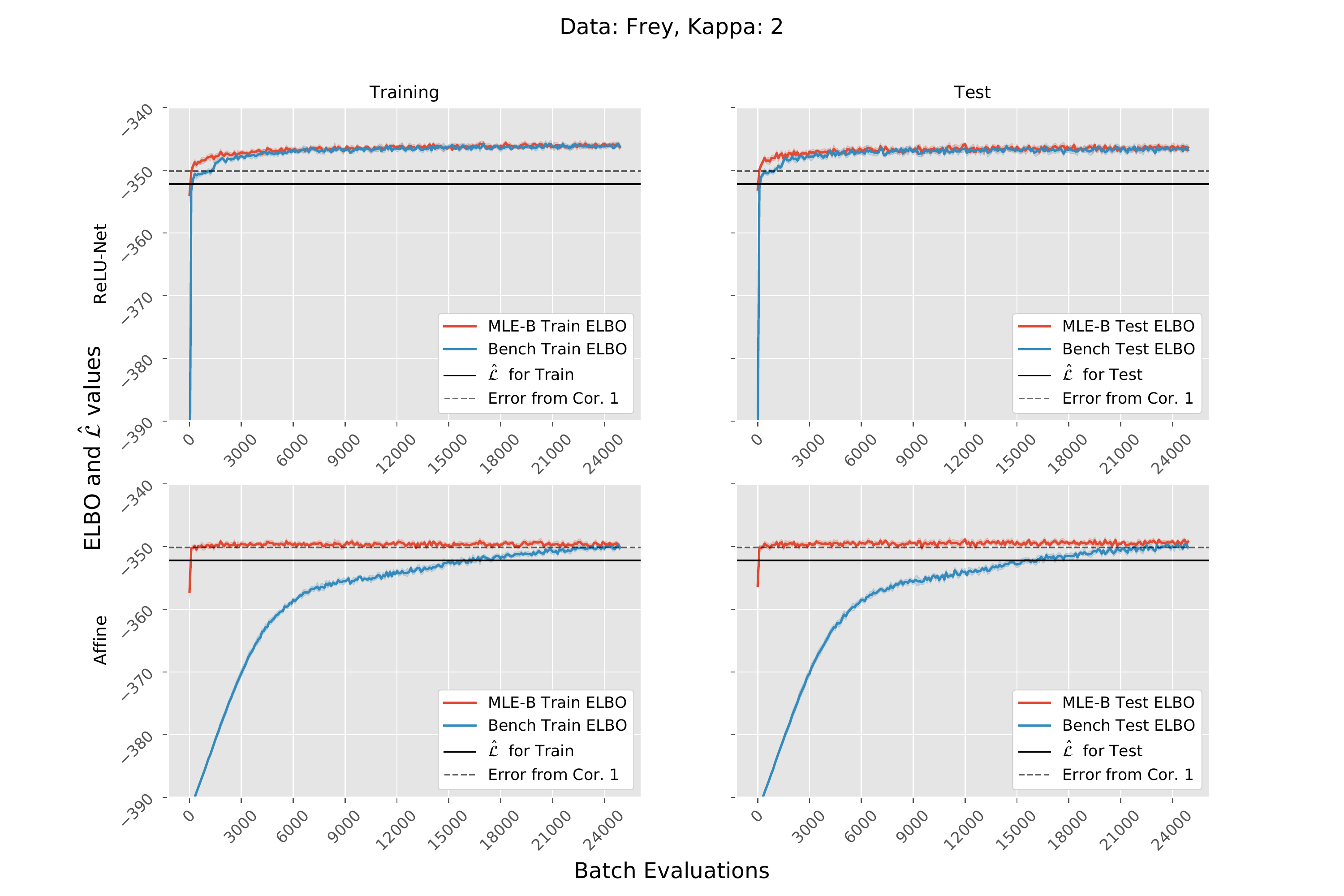}}
		\caption{The figure shows two different setups ReLU-Net and Affine with frey data, $\kappa=2$ and sigmoid activation. Displayed are the ELBOs of both initialisations MLE-B and Bench as  well as the lower bound $\hELBO$ and the expected error \added{$\E_{q_{\hat{\phi}}}[R_2(\hat{\theta})]$ }as provided by Corollary \ref{corr: F error}, calculated based on MLE. On the left the \added{ELBO} values are calculated with the training data and on the right with test data. \added{The curves are based on simulations with 10 different seeds. We display the average training performances with pointwise 0.95 confidence intervals.}}
		\label{fig:data_frey_kappa_2_act_sigmoid_plot}
	\end{center}
	\vskip -0.35in
\end{figure*}

In Figure \ref{fig:data_frey_kappa_2_act_sigmoid_plot}, the bound $\hELBO$ is reasonable and both architectures do not perform significantly better. 
The results of all simulation schemes can be found in Appendix \ref{app:Simulation results}. Comparing these simulation results and considering Figure \ref{fig:data_frey_kappa_2_act_sigmoid_plot}, we observe the following:
\begin{itemize}
	\item For the affine decoder architecture, the initialization MLE-B converges directly, whereas the Benchmark takes much more time. The end values are comparable. For the ReLU-Net decoder architecture, the performance of the two initialization methods mostly shows a small initial advantage of MLE-B which, however, is not as clear as for the affine architecture.
	\item In no simulation setup a net was over-fitting, not even for large values of $\kappa$ with synthetic data, where a much smaller $\kappa$ \replaced{is}{was} needed. This is a consequence of the auto-pruning.
	\item \added{For the MLE values, based on Corollary \ref{corr: F error} we know that $\ELBO(\hat{\theta},\hat{\phi})$ lies above $\hELBO(\hat{\theta},\hat{\phi})$. }It seems that MLE-B needs a very short burn-in period to perform according to Corollary \ref{corr: F error}. We believe that the offset at the beginning originates from not readily initialized hidden layers (mainly for the encoder). \added{We can observe a performance above the expected error. Possible reasons for this are that the sampled values during optimization differ from the analytically calculated expectation and different realized values for $\theta$ and $\phi$. }
\end{itemize}

\subsection{Latent dimension activities}\label{subsec:Latent dimension activities}

In this section, we show that the analytical activities in \eqref{eq:activity analytical} serve as a good indicator for the amount of active nodes without conducting training. This statement also holds for ReLU-Net decoder and not just for the affine case. We consider the mnist data set. For the Gaussian observation model, Figure \ref{fig:Activities_21_26} shows \added{histograms of} the \replaced{activity statistics $A_{z_j}$, proposed by \mbox{\citet{Burda2015}},}{calculated activities} for the analytical case and an Affine / ReLU-Net decoder (Details on the architectures can be found in Appendix \ref{app:Architecture}) after the training. Displayed are the values for different $\beta$. \added{Table \ref{tab:Normal Histogram distance} displays the calculated distances to the analytical calculations based on 10 simulations. 

}   \replaced{The corresponding figure and table for the}{The} Bernoulli observation model can be found in Appendix \ref{app:Bernoulli Activities}.

\begin{figure*}[!htb]
	\begin{center}
		\centerline{\includegraphics[width=\linewidth]{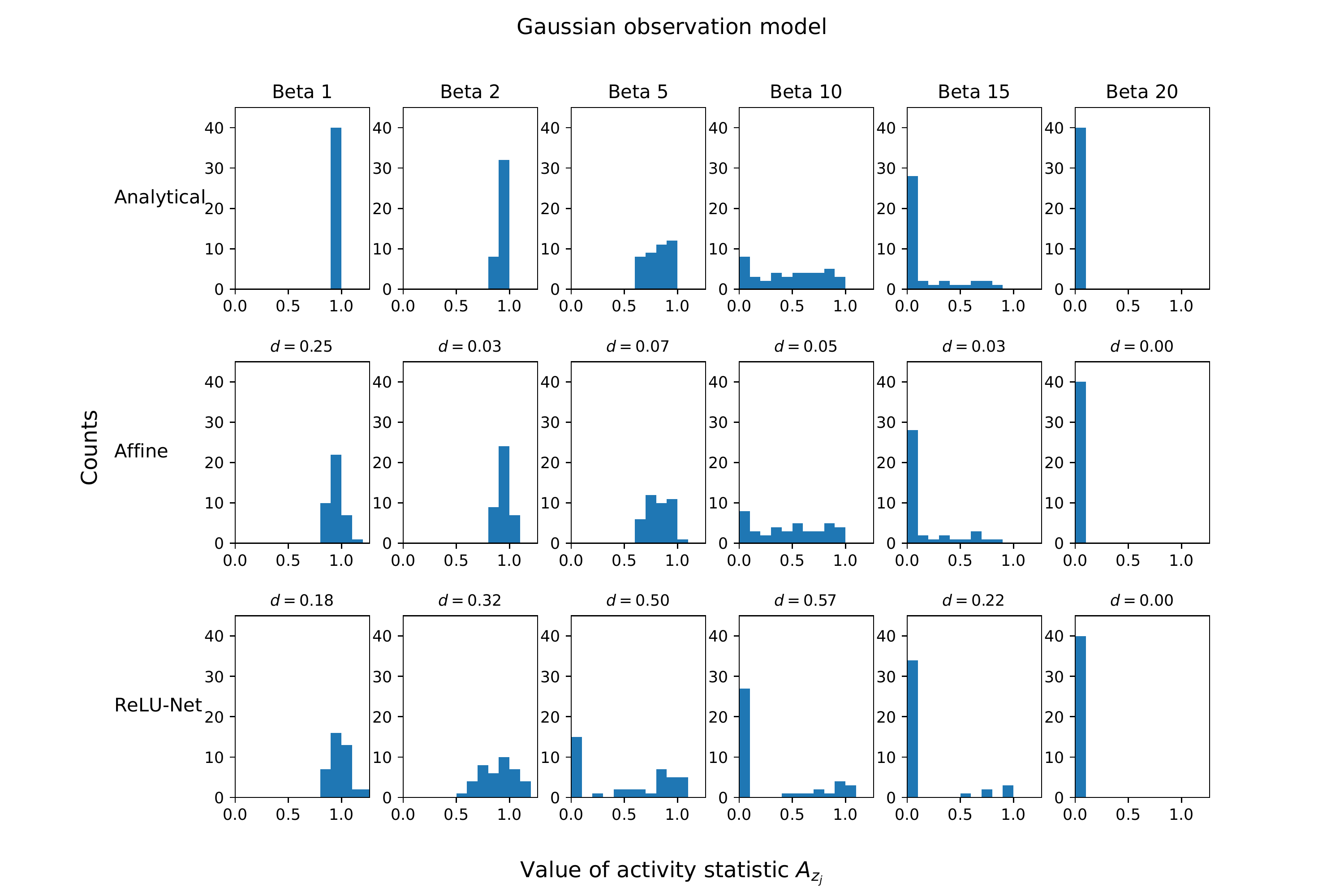}}
		\caption{The figure shows\added{ histograms of} the activities for 40 latent dimensions for our analytical calculation and an Affine/ ReLU-Net decoder (as described in appendix \ref{app:Architecture}) after training. \replaced{We have considered}{Considered is} the mnist data set with a Gaussian observation model. \added{Above each Affine and ReLU-Net histogram plot, we show the distance ($\in [0,1]$, lower is better) as defined in \eqref{eq:histogramm distance} to the analytical histogram. }}
		\label{fig:Activities_21_26}
	\end{center}
	\vskip -0.35in
\end{figure*}

{\renewcommand{\arraystretch}{0.5}
	\begin{table}[!htb]
		\caption{\added{We display the distances ($\in [0,1]$, lower is better) of the histograms to the corresponding analytical calculation for the Gaussian observation model. Displayed are the results of 10 simulations as ``mean$\pm$std''.}}
		\label{tab:Normal Histogram distance}
		\centering
		\begin{tabular}{ccccccc}
			\toprule
			 & Beta 1 & Beta 2 &Beta 5 &Beta 10 &Beta 15 & Beta 20 \\
			 \midrule
			Affine & 0.24 ($\pm$0.06) & 0.04 ($\pm$0.02) & 0.04 ($\pm$0.02) & 0.05 ($\pm$0.01) & 0.03 ($\pm$0.02) & 0 ($\pm$0) \\
			ReLU & 0.18 ($\pm$0.03) & 0.3 ($\pm$0.02) & 0.49 ($\pm$0.04) & 0.55 ($\pm$0.03) & 0.23 ($\pm$0.01) & 0 ($\pm$0) \\ 
			\bottomrule
		\end{tabular}
\end{table}}

Since the analytical calculations are based on the affine decoder architecture, the first two rows look similar. Given a value of $\beta$ we can make trustworthy predictions how much latent dimensions become inactive during training. 

The ReLU-Net decoder behaves \deleted{slightly }differently. \added{It seems to be that the deeper structure and piecewise linear functions allow the model to use less latent dimensions to properly model the data distribution and hence more latent dimensions can become inactive. }
\replaced{Given this point of view, we can use the analytical calculation of the statistics as a lower estimate of }{But still, given our analytical calculations we can expect }how much latent dimensions \deleted{approximately}\added{will} turn out to be inactive after training. 

\added{Since the analytical calculation of the statistic in \eqref{eq:activity analytical} is low cost, we recommend to use it in either case. }

\section{Conclusion}\label{sec:Conclusion}

We have established a new framework for $\beta$-VAE, by interpreting the decoder of a $\beta$-VAE as a GLM. Given this framework, we derive and analyse an approximation to the $\beta$-VAE objective based on the EDF observation model. 

We derive MLE for this approximation in the affine transformation setting. Furthermore, we present simulation results validating the theory on real world data sets, like the frey and mnist data set. The results here generalize previous work in this field. 

\replaced{Further, we}{We further} provide an analytical description of the auto-pruning of $\beta$-VAE \replaced{. We show that the parameter MLEs are directly influenced by the cut-off term in \eqref{eq:cut-off}, which yields}{ and explain} the dependence on the parameter $\beta$ for the affine decoder setting. \added{Furthermore, the amount of active units is directly affected by this term.} Our simulation results suggest that the implications can be used for ReLU-Net decoders.

A possible extension is to integrate distributions like the Gamma distribution which belongs to the EDF.

%


\appendix
\section{Appendix: Proofs}

\subsection{Auxiliary results}

\begin{lemm}\label{lemm: Sigma is minimum}
	Let $B,\Gamma \in \R^{\kappa \times \kappa}$ be symmetric positive definite matrices. 
	Then it holds
	\[
	B = \arg\min_{\Gamma \succ 0}  \tr(B \Gamma^{-1}) + \log |\Gamma|
	\]
	and hence
	\[
	\kappa  + \log |B| = \min_{\Gamma \succ 0}  \tr(B \Gamma^{-1}) + \log |\Gamma|.
	\] 
\end{lemm}
\begin{proof}
	Define the two distributions $\Nor_0(\bmu,B)$ and  $\Nor_1(\bmu,\Gamma)$. We have
	\begin{align}\label{max_ent}
	&2 \cdot \KL\left(\Nor_0(\bmu,B)|| \Nor_1(\bmu,\Gamma)\right) \nonumber\\
	&= \tr(B \Gamma^{-1}) + \log |\Gamma| - \kappa -  \log |B|. 
	\end{align}
	Now, consider that for the Kullback-Leibler-Divergence with probability distributions $P$ and $Q$ it holds:
	\begin{itemize}
		\item $\KL\left(  P || Q \right) \geq 0$ for all inputs.
		\item $\KL\left( P || Q\right) = 0$ if and only if $P=Q$ almost everywhere.
	\end{itemize}
	Hence, we conclude $B=\Gamma$ in the minimum.
\end{proof}

\subsection{Proof of Proposition \ref{prop:general prop for VAE target}}\label{app:proof of prop}
\begin{proof}
	To proof Proposition \ref{prop:general prop for VAE target}, we change the perspective. Instead of maximizing, we want to minimize the negative expression given by
	\begin{align}\nonumber
	-\ELBO(\phi,\theta) :=& \dfrac{1}{N} \sum_{i=1}^N \beta \cdot \KL\left(q_{\phi}(Z|\ix)|| P(Z)\right) \\
	&- \E_{Z \sim q_{\phi}\left(\cdot|\ix\right)}\left[\log P_{\bvartheta,\varphi}(\ix|Z)\right] .\label{eq:bin-ELBO}
	\end{align}
	
	Looking at \eqref{eq:bin-ELBO}, we see two terms.
	For the KL-\replaced{D}{d}ivergence we have that
	\begin{align}\label{KL_equiv}
	&2 \cdot \KL\left(q_{\phi}(Z|\ix)|| P(Z)\right) \nonumber\\
	&= \tr[\bSigmazi] - \log |\bSigmazi| + ||\bmuzi||_2^2 - \kappa
	\end{align}
	and for the second term (with $q_{\phi}$ as abbreviation for $q_{\phi}(\cdot|\ix)$) we get with a second-order Taylor Expansion as in \eqref{eq:taylor PooleKumar} and
	\[
	\E\left[XX^T\right] = \textnormal{Cov}\left(X\right) + \bmu \bmu^T
	\]
	that
	\begin{align}
	&- \E_{q_{\phi}}\left[\log P_{\bvartheta,\varphi}(\ix|Z)\right] \nonumber\\
	&\approx -\log P_{\theta}(x | \neyzi) - \E_{q_{\phi}}\left[ J_{f_x}(\neyzi) ( z- \neyzi) + \dfrac{1}{2}( z- \neyzi)^T H_{f_x}(\neyzi)( z- \neyzi)\right]\nonumber\\
	&= -\log P_{\theta}(x | \neyzi) -  J_{f_x}(\neyzi) ( \bmuzi- \neyzi)\nonumber \\
	&- \dfrac{1}{2} \tr\left(H_{f_x}(\neyzi)\bSigmazi\right)- \dfrac{1}{2}\tr\left(H_{f_x}(\neyzi)\left(\bmuzi - \neyzi\right)\left(\bmuzi-\neyzi\right)^T\right).\label{eq:problemstelle_ELBO}
	\end{align}

	Putting the two terms together, for our target function \eqref{eq:bin-ELBO}, it follows that it is approximated by
	\begin{align*}
	-\hELBO(\phi,\theta):= &\dfrac{1}{N} \sum_{i=1}^N \Bigg[ \dfrac{\beta \tr[\bSigmazi]}{2} - \dfrac{ \beta \log |\bSigmazi|}{2} + \dfrac{\beta ||\bmuzi||_2^2}{2}- \dfrac{\beta \kappa}{2}- \dfrac{1}{2} \tr\left(H_{f_x}(\neyzi)\bSigmazi\right)\\
	&-\log P_{\theta}(x | \neyzi) -  J_{f_x}(\neyzi) ( \bmuzi- \neyzi)\nonumber \\
	&- \dfrac{1}{2}\tr\left(H_{f_x}(\neyzi)\left(\bmuzi - \neyzi\right)\left(\bmuzi-\neyzi\right)^T\right)\Bigg].
	\end{align*}
	
	All potential minima w.r.t. $\bSigmazi$ have to conform to 
	\begin{equation}
	\bhSigmazi = \left(I_{\kappa} - \dfrac{1}{\beta}H_{f_x}(\neyzi)\right)^{-1},
	\end{equation}
	independent of $\ix$. Note that this expression is well-defined as per assumption, we have $\beta > 0$ and further it can be shown that $-H_{f_x}(\neyzi)$ is positive semi definite with $X$ having an EDF distribution. 
	To see that these are minima and not maxima, consider Lemma \ref{lemm: Sigma is minimum},
	\[
	\kappa + \log |A A^T|= \min_{\Gamma \succ 0} \tr(AA^T \Gamma^{-1}) + \log |\Gamma|,
	\]
	where we set $\Gamma^{-1} = \bhSigmazi$ and $AA^T = I_{\kappa} -\dfrac{1}{\beta}H_{f_x}(\neyzi)$. 
	
		Given the fact, that $\bvartheta$ is a piecewise linear function, we have 
	\begin{equation}\label{eq:H_fx pw lin rep}
	H_{f_x}(z) = \Jv(\neyzi)^T \left(\nabla_\vartheta^2 \log P_{\theta}\left(x | z\right)\right)\Jv(\neyzi),
	\end{equation}
	as in \eqref{eq:pw Lin Hf PooleKumar}. Further, since we assume a $\neyzi \in ker(\bvartheta)$, it can be shown that $\nabla_\vartheta^2 \log P_{\theta}\left(x | z\right)$ is independent of the explicit choice of $\neyzi$, so we can use	\begin{equation}\label{eq:H_fx Gamma}
	HH^T=\Gamma := -\left(\nabla_\vartheta^2 \log P_{\theta}\left(x | \neyzi \right)\right),
	\end{equation}
	which is positive definite when the distribution of X given Z belongs to the EDF. If follows 
	\begin{equation}\label{opt-bSigmaz}
	\bhSigmazi = \left(I_{\kappa} + \dfrac{1}{\beta}\Jv(\neyzi)^T \Gamma\Jv(\neyzi)\right)^{-1}.
	\end{equation}
	
	Given $\bhSigmazi$ the minimal target function evaluated at this point becomes  
	\begin{align*}
	-\hELBO(\phi \setminus \{\bSigma_z\},\theta)
	= &\dfrac{1}{N} \sum_{i=1}^N \Bigg[ \dfrac{\beta ||\bmuzi||_2^2}{2}-\log P_{\theta}(x | \neyzi) -  J_{f_x}(\neyzi) ( \bmuzi- \neyzi)\nonumber \\
	&- \dfrac{1}{2}\tr\left(H_{f_x} (\neyzi) \cdot \left(\bmuzi - \neyzi\right)\left(\bmuzi-\neyzi\right)^T\right)\\
	&+ \dfrac{1}{2}\log \left|{\bhSigmazi}^{-1}\right|\Bigg] .
	\end{align*}
	For $\bmuzi$, we get as minimal points 
	\begin{align}
	\bhmuzi &= \dfrac{1}{\beta}\bhSigmazi \left(J_{f_x}(\neyzi)^T -  H_{f_x}(\neyzi) \cdot \neyzi \right) \label{opt-bmuzi}.
	\end{align}
	The candidates for an optimal $\bhmuzi$ are minima since the second derivative is a positive constant times ${\bhSigmazi}^{-1}$, which is positive definite. Given the optimal $\bmuzi$ and $\bSigmazi$ our target function is only dependent on the parameters $\theta$. We get 
	\begin{align}
	-\hELBO(\theta) =
	&\dfrac{1}{N} \sum_{i=1}^N \Bigg[\left(J_{f_x}(\neyzi)^T -  H_{f_x} (\neyzi) \cdot \neyzi \right)^T E \left(J_{f_x}(\neyzi)^T -  H_{f_x} (\neyzi) \cdot \neyzi\right)\label{eq:SVD-prob}\\
	& -\log P_{\theta}(x | \neyzi) - \dfrac{1}{2} {\neyzi}^T H_{f_x} (\neyzi) \neyzi + J_{fx}(\neyzi) \neyzi + \dfrac{\beta}{2}\log \left|{\bhSigmazi}^{-1}\right|\Bigg], \nonumber
	\end{align}
	where $E:= \dfrac{1}{2\beta}{\bhSigmazi}^2 - \dfrac{1}{2\beta^2}\bhSigmazi H_{f_x} (\neyzi)\bhSigmazi - \dfrac{1}{\beta}\bhSigmazi$. Consider a Singular Value Decomposition of $H^T\Jv(\neyzi) = U \widetilde{D} V^T$ with $U\in \R^{d\times d}$ and $V\in \R^{\kappa\times \kappa}$ are unitary matrices and 
	\[
	\widetilde{D} = \begin{bmatrix}
	\delta_1 & \ldots & 0 \\
	\vdots & \ddots & \vdots \\
	0 & \ldots & \delta_{\kappa} \\
	\vdots & \ddots & \vdots \\
	0 & \ldots & 0 \\
	\end{bmatrix}\in \R^{d\times \kappa}.
	\]
	
	We have 
	\[
	\widetilde{D}^T \widetilde{D} = \begin{bmatrix}
	\delta_1^2 & \ldots & 0 \\
	\vdots & \ddots & \vdots \\
	0 & \ldots & \delta_{\kappa}^2 \\
	\end{bmatrix}
	\]
	and can write 
	\begin{align*}
	\bhSigmazi =&\Big(V (I_{\kappa} + \dfrac{1}{\beta} \widetilde{D}^T \widetilde{D})  V^T\Big)^{-1} \\
	&= V \widehat{D} V^T,
	\end{align*}
	with $\widehat{D} := \diag\left(\dfrac{1}{1 + \beta^{-1} \delta_1^2} ,\ldots, \dfrac{1}{1 + \beta^{-1} \delta_{\kappa}^2}\right)$. For $E$ it follows that 
	\begin{align}
	&\dfrac{1}{2\beta}\left[{\bhSigmazi}^2 - \dfrac{1}{\beta}\bhSigmazi H_{f_x} (\neyzi)\bhSigmazi - 2\bhSigmazi\right] \nonumber\\
	&=\dfrac{1}{2\beta}V\left[\widehat{D}^2 + \dfrac{1}{\beta}\widehat{D}\widetilde{D}^T \widetilde{D}\widehat{D} - 2\widehat{D}\right]V^T \nonumber \\
	&=\dfrac{-1}{2\beta}V\widehat{D}V^T. \label{SVD-prob-part1} 
	\end{align}
	The justification of the last equation becomes apparent, when we consider one respective diagonal element $\delta_{\cdot}$ of the diagonal matrices in the equation. We have
	\begin{align*}
	&\dfrac{1}{(1 + \beta^{-1} \delta_{\cdot}^2)^2}  +\dfrac{\beta^{-1}\delta_{\cdot}^2}{(1 + \beta^{-1} \delta_{\cdot}^2)^2} - \dfrac{2}{1 + \beta^{-1} \delta_{\cdot}^2}\\  
	&= \dfrac{1  + \beta^{-1} \delta_{\cdot}^2- 2(1 + \beta^{-1} \delta_{\cdot}^2)}{2(1 + \beta^{-1} \delta_{\cdot}^2)^2}\\
	&=\dfrac{-1}{(1 + \beta^{-1} \delta_{\cdot}^2)}.
	\end{align*}
	
	We can further rephrase a part of \eqref{eq:SVD-prob}, as we can use \eqref{eq:H_fx pw lin rep},\eqref{eq:H_fx Gamma} and have
	\[
	J_{fx}(\neyzi) = \left(\nabla_\vartheta\log P_\theta\left(\ix|\neyzi\right)\right) \Jv(\neyzi),
	\]
	with $\iy:=\left(\nabla_\vartheta\log P_\theta\left(\ix|\neyzi\right)^T+\Gamma\Jv(\neyzi)\neyzi\right)$ we have
	\begin{align*}
	&- \dfrac{1}{2} {\neyzi}^T H_{f_x} (\neyzi) \neyzi + J_{fx}(\neyzi) \neyzi \\
	&=\dfrac{1}{2} {\neyzi}^T \Jv(\neyzi)^T H H^T \Jv(\neyzi) \neyzi + \left(\nabla_\vartheta\log P_\theta\left(\ix|\neyzi\right)\right) H^{-T} H^T\Jv(\neyzi) \neyzi \\
	&= \dfrac{1}{2} {\iy}^T\Gamma^{-1} \iy - \dfrac{1}{2}\nabla_\vartheta\log P_\theta\left(\ix|\neyzi\right)\Gamma^{-1}\nabla_\vartheta\log P_\theta\left(\ix|\neyzi\right)^T.
	\end{align*}
	
	Using this and \eqref{SVD-prob-part1}, for \eqref{eq:SVD-prob} we get
	\begin{align*}
	-\hELBO(\theta) =
	&\dfrac{1}{N} \sum_{i=1}^N \Bigg[{\iy}^T \left(\dfrac{-1}{2\beta} \Jv(\neyzi) V\widehat{D}V^T \Jv(\neyzi)^T + \dfrac{1}{2}\Gamma^{-1}\right)\iy -\log P_{\theta}(x | \neyzi) \\ 
	&-\dfrac{1}{2}\nabla_\vartheta\log P_\theta\left(\ix|\neyzi\right)\Gamma^{-1}\nabla_\vartheta\log P_\theta\left(\ix|\neyzi\right)^T + \dfrac{\beta}{2}\log \left|{\bhSigmazi}^{-1}\right|\Bigg].
	\end{align*}
	
	Since $\Gamma = HH^T$, we can further rewrite 
	\begin{align*}
	&\dfrac{-1}{2\beta} \Jv(\neyzi) V\widehat{D}V^T \Jv(\neyzi)^T + \dfrac{1}{2}\Gamma^{-1}\\
	&= \dfrac{1}{2} H^{-T} U\left[-\beta^{-1} \widetilde{D}\widehat{D}\widetilde{D}^T + I \right] U^T H^{-1} \\
	&= \dfrac{1}{2} H^{-T} U\left[\beta^{-1} \widetilde{D}\widetilde{D}^T + I \right]^{-1} U^T H^{-1} \\
	&= \dfrac{1}{2} \Gamma^{-1}\left[\Gamma^{-1} + \beta^{-1} \Jv(\neyzi)\Jv(\neyzi)^T \right]^{-1} \Gamma^{-1}\\
	\end{align*}
	and together with
	\begin{align*}
	\log\left|{\bhSigmazi}^{-1}\right|  &= \log\left|I_\kappa + \beta^{-1} \widetilde{D}^T\widetilde{D}\right|= \log\left|I_d + \beta^{-1} \widetilde{D}\widetilde{D}^T\right|\\
	&=\log\left|\Gamma^{-1} + \beta^{-1} \Jv(\neyzi)\Jv(\neyzi)^T\right| + \log\left|\Gamma\right|,
	\end{align*}
	 for $C(\neyzi):=\Gamma^{-1} + \beta^{-1} \Jv(\neyzi)\Jv(\neyzi)^T$ we get for our target function
	\begin{align*}
	-\hELBO(\theta) =
	&\dfrac{1}{N} \sum_{i=1}^N \Bigg[\dfrac{1}{2}{\iy}^T \Gamma^{-1}C(\neyzi)^{-1}\Gamma^{-1}\iy -\log P_{\theta}(x | \neyzi) \\ 
	&-\dfrac{1}{2}\nabla_\vartheta\log P_\theta\left(\ix|\neyzi\right)\Gamma^{-1}\nabla_\vartheta\log P_\theta\left(\ix|\neyzi\right)^T \\
	&+ \dfrac{\beta}{2}\log \left|C(\neyzi)\right|\Bigg] + \dfrac{\beta}{2}\log\left|\Gamma\right|.
	\end{align*}
	Based on the assumption, that the distribution of X given Z belongs to the EDF and $\neyzi \in ker(\bvartheta)$, we have
	\begin{align*}
	-\log P_{\theta}(x | \neyzi) =& d/\varphi F(0) - \sum_{j=1}^{d} K(\ix_j,\varphi),\\
	\nabla_\vartheta\log P_\theta\left(\ix|\neyzi\right)^T =& \varphi^{-1}\left(\ix - F'(0)\right)\\
	&\\
	\Gamma^{-1} =& F''(0)^{-1} \varphi  \bI_d
	\end{align*}
	and hence
	\[
	\Gamma^{-1} \iy = F''(0)^{-1} \left(\ix - F'(0)\right) + \Jv(\neyzi)\neyzi.
	\]
	Therefore, we get
	\begin{align*}
	-\hELBO(\theta) =
	\dfrac{1}{N} \sum_{i=1}^N \Bigg[&\dfrac{1}{2}\left(F''(0)^{-1}(\ix-F'(0)) + \Jv(\neyzi)\neyzi \right)^T C(\neyzi)^{-1}\\
	&\quad\quad\left(F''(0)^{-1}(\ix-F'(0)) + \Jv(\neyzi)\neyzi \right)\\
	&+ \dfrac{\beta}{2}\log \left|C(\neyzi)\right|+ \dfrac{\beta \cdot d}{2}\log \left(\varphi^{-1} F''(0)\right) +\dfrac{1}{2}D\left(\varphi\right)\Bigg],
	\end{align*}
	
	with 
	\[
	C(\neyzi) = F''(0)^{-1} \varphi  \bI_d + \beta^{-1} \Jv(\neyzi)\Jv(\neyzi)^T
	\]
	and
	\begin{equation}\label{eq:definition of D(varphi)}
	D\left(\varphi\right) :=  \dfrac{2 d}{\varphi} F(0) - \dfrac{1}{N} \sum_{i=1}^N  \left[\dfrac{1}{ F''(0) \varphi} \left\|\ix - F'(0)\right\|_2^2 + 2 \sum_{j=1}^{d} K(\ix_j,\varphi)\right].
	\end{equation}
\end{proof}

\subsection{Proof of Corollary \ref{corr: F error}}\label{app:proof of corollary F error}

\begin{proof}
	We look at the Gaussian, Binomial and Poisson cases separately.
	\begin{itemize}
		\item Gaussian case:
		As for the EDF representation of a Gaussian model we have $F(\vartheta)=\dfrac{\vartheta^2}{2}$. As $\bvartheta$ is assumed to be a piecewise linear function, it follows directly that all third or higher derivatives of $f_x(z)$ vanish and we are done.		
		\item Binomial case:
		Given $\bvartheta$ is a p.w. linear function we have for $u,v,m,n \in \{1,\ldots,\kappa\}$
		\begin{equation}\label{eq:third derivative Binomial case}
		\dfrac{\partial^3 f_x(z)}{\partial z_u\partial z_v\partial z_m} = \sum_{j=1}^{d}-F^{(3)}(\bvartheta_j(z))\dfrac{\partial \bvartheta_j(z)}{\partial z_u}\dfrac{\partial \bvartheta_j(z)}{\partial z_v}\dfrac{\partial \bvartheta_j(z)}{\partial z_m}
		\end{equation}
		and 
		\begin{equation}\label{eq:fourth derivative Binomial case}
		\dfrac{\partial^4 f_x(z)}{\partial z_u\partial z_v\partial z_m\partial z_n} = \sum_{j=1}^{d}-F^{(4)}(\bvartheta_j(z))\dfrac{\partial \bvartheta_j(z)}{\partial z_u}\dfrac{\partial \bvartheta_j(z)}{\partial z_v}\dfrac{\partial \bvartheta_j(z)}{\partial z_m}\dfrac{\partial \bvartheta_j(z)}{\partial z_n}.
		\end{equation}

		For $\neyzi \in ker(\bvartheta)$, we have $F^{(3)}(\bvartheta_j(z))=0$ and can conclude that for these points a second order Taylor approximation is the same as a third order approximation. The remainder is given in the Lagrange form by
		\[
		R_2(z,\neyzi) =R_3(z,\neyzi) = \sum_{j=1}^{d} \dfrac{-F^{(4)}(\bvartheta_j(\xi))}{4!} \left(\sum_{u=1}^{\kappa}\dfrac{\partial \bvartheta_j(\xi)}{\partial z_u} (z-\neyzi)_u\right)^4,
		\]
		with $\xi = \neyzi + c \cdot (z-\neyzi)$, where $c\in [0,1]$. We can rewrite $\sum_{u=1}^{\kappa}\dfrac{\partial \bvartheta_j(\xi)}{\partial z_u} (z-\neyzi)_u = {\Jv}_j(\xi)(z-\neyzi)$ and we can show that
		\[
		-F^{(4)}(\bvartheta_j(\xi)) \in [-n/24,n/8].
		\] 
		Using this, the first statement follows for the Binomial case follows.
		
		For the second statement we write the remainder in the Lagrange form for a second order-approximation and get
		\begin{align*}
		R_2(z,\neyzi) &=\sum_{j=1}^{d} \dfrac{-F^{(3)}(\bvartheta_j(\xi))}{3!} \left(\sum_{u=1}^{\kappa}\dfrac{\partial \bvartheta_j(\xi)}{\partial z_u} (z-\neyzi)_u\right)^3\\
		&=\sum_{j=1}^{d} \dfrac{-F^{(3)}(\bvartheta_j(\xi))}{3!} \left({\Jv}_j(\xi)(z-\neyzi)\right)^3,
		\end{align*}
		
		If we assume $\bvartheta$ to be an affine transformation on the convex set spanned by $z$ and $\neyzi$ we have 
		\[
		\bvartheta(\xi) = W \xi + b
		\]
		for some $W \in \R^{d \times \kappa}$ and $b \in \R^d$. 
		Hence we can rewrite the equation above as 
		\[
		R_2(z,\neyzi) =\sum_{j=1}^{d} \dfrac{-F^{(3)}(W_j \neyzi + b_j + c \cdot W_j (z-\neyzi))}{3!} \left(W_j(\xi)(z-\neyzi)\right)^3
		\]
		Since $\neyzi \in ker(\bvartheta)$, we can write 
		\[
		W_j \neyzi + b_j + c \cdot W_j (z-\neyzi) = c \bvartheta_j(z).
		\]
		Notice that 
		\[
		F(\vartheta)^{(3)} =-n \dfrac{e^\vartheta\left( e^{\vartheta}- 1\right)}{\left(e^\vartheta + 1\right)^3}
		\]
		is point symmetric to zero, with negative values if $\vartheta>0$ and positive values if $\vartheta<0$. Therefore, we can write 
		\[
		R_2(z,\neyzi) =\sum_{j=1}^{d} \dfrac{-|c \bvartheta_j(z)|}{3!} \left|\bvartheta_j(z)\right|^3 \geq 0.
		\]
		This yields the statement.
		
		\item Poisson case:
		We have $F(\vartheta)= \exp(\vartheta)$ and can write the remainder in Lagrange form as
		\[
		R_2(z,\neyzi) = \sum_{j=1}^{d}\dfrac{-\exp(\bvartheta_j(\xi))}{6} \left({\Jv}_j(\xi)(z-\neyzi)\right)^3,
		\]
		with $\xi = \neyzi + c \cdot (z-\neyzi)$, where $c\in [0,1]$.
		Since we assume $\bvartheta$ to be an affine transformation on the convex set spanned by $z$ and $\neyzi$ and $\neyzi \in ker(\bvartheta)$, like in the Binomial case we have
		\[
		R_2(z,\neyzi) = \sum_{j=1}^{d}\dfrac{-\exp(c \bvartheta_j(z))}{6} \left(\bvartheta_j(z)\right)^3,
		\]
		for some $W \in \R^{d \times \kappa}$ and $b \in \R^d$.

		Consider the two cases $\bvartheta_j(z)< 0$ and $\bvartheta_j(z) \geq 0$. For the case $\bvartheta_j(z)<0$ we have
		
		\begin{align*}
			&\min_{c\in[0,1]} \left(\dfrac{\exp(-c |\bvartheta_j(z)|)}{6} \left|\bvartheta_j(z)\right|^3\right)&\\
			&=\min_{c\in[0,1]} \left(\exp(-c |\bvartheta_j(z)|)\right) \left|\bvartheta_j(z)\right|^3/6&\\
			&=\min_{c\in[0,1]} \left(\dfrac{1}{\exp(c |\bvartheta_j(z)|)}\right) \left|\bvartheta_j(z)\right|^3/6&\\
			&= \left(\dfrac{1}{\max_{c\in[0,1]}\exp(c |\bvartheta_j(z)|)}\right) \left|\bvartheta_j(z)\right|^3/6&\\
			&= \left(\dfrac{1}{\exp( |\bvartheta_j(z)|)}\right) \left|\bvartheta_j(z)\right|^3/6&
		\end{align*}
		
		and 
		
		\begin{align*}
		&\max_{c\in[0,1]} \left(\dfrac{\exp(-c |\bvartheta_j(z)|)}{6} \left|\bvartheta_j(z)\right|^3\right)&\\
		&= \left(\dfrac{1}{\min_{c\in[0,1]}\exp(c |\bvartheta_j(z)|)}\right) \left|\bvartheta_j(z)\right|^3/6&\\
		&= \left|\bvartheta_j(z)\right|^3/6&
		\end{align*}

		and hence
		
		\begin{equation}\label{eq:result for poisson neg}
		\sum_{j=1}^d -\bvartheta_j(z)^3 \cdot \exp( \bvartheta_j(z)/6 \leq  R_2(z,\neyzi) \leq \sum_{j=1}^d -\bvartheta_j(z)^3/6.
		\end{equation}
		For the case $\bvartheta_j(z)\geq0$ we have
		
		\begin{align*}
		&\min_{c\in[0,1]} \left(\dfrac{-\exp(c |\bvartheta_j(z)|)}{6} \left|\bvartheta_j(z)\right|^3\right)&\\
		&=\max_{c\in[0,1]} \left(\exp(c |\bvartheta_j(z)|)\right) \cdot - \left|\bvartheta_j(z)\right|^3/6&\\
		&=\exp(|\bvartheta_j(z)|)\cdot - \left|\bvartheta_j(z)\right|^3/6&
		\end{align*}
		
		as well as 
		
		\begin{align*}
		&\max_{c\in[0,1]} \left(\dfrac{-\exp(c |\bvartheta_j(z)|)}{6} \left|\bvartheta_j(z)\right|^3\right)&\\
		&=\min_{c\in[0,1]} \left(\exp(c |\bvartheta_j(z)|)\right) \cdot - \left|\bvartheta_j(z)\right|^3/6&\\
		&=- \left|\bvartheta_j(z)\right|^3/6&
		\end{align*}
		and it follows again \eqref{eq:result for poisson neg}.
	\end{itemize}
\end{proof}

\section{Simulation}\label{app:Simulation}

\subsection{Architecture, latent dimension \texorpdfstring{$\kappa$}{TEXT} and the last decoder activation}\label{app:Architecture}

For the architecture, we look at two different versions, which we denote as ``deep'' and ``canonical''. The deep architecture is given as in \citet{Dai2018}, by 
\begin{align*}
x(d)\rightarrow E_1(2000)\rightarrow E_2(1000)&\rightarrow \; \bmu_z(\kappa)\; \rightarrow D_1(1000)\rightarrow D_2(2000)\rightarrow \hat{x}(d),\\
&\searrow\log\bsigma_z^2(\kappa)\nearrow& \\
\end{align*}
where $E$/$D$ denote encoder/decoder layers and the values in the brackets indicate the dimension of the layer. So, $\kappa$ is the dimension of the latent space and we use the values $2,5$ and $20$ for different simulation setups. The covariance of the variational distribution is set as diagonal matrix. 

The canonical architecture is given by 
\begin{align*}
x(d)\rightarrow E_1(2000)\rightarrow E_2(d)&\rightarrow \; \bmu_z(\kappa)\; \rightarrow\;\hat{x}(d).\\
&\searrow\log\bsigma_z^2(\kappa)\nearrow& \\
\end{align*}
The canonical architecture conforms to the assumptions of Proposition \ref{prop:general prop for VAE target}.

The hidden layers of the encoder and the decoder are implemented with ReLU-activation (see \citealt{nair2010rectified}), which is known to be highly expressive. The ``$\bmu_z$''- and ``$\log\bsigma_z^2$''- layer have linear activations. The last layer ``$\hat{x}$'' has either a linear  (Gaussian observation model) or a sigmoid (Bernoulli observation model) activation as reported in Table \ref{tab:comb_exp_act}.

Apart from the fact that we expect both architectures to provide a better loss than provided by our theoretical bound, it should be able to represent the optimal $\bhmuzi$ and diagonal entries of the optimal $\bhSigmazi$ in \eqref{eq:optSigma with opt W} and \eqref{eq:optMu with opt W} if necessary. 

\subsection{Data}

We consider three data sets: a synthetic data set (we describe the construction at the end of this chapter), the mnist data set (see \citealt{LeCun2010}) and the frey data set.\footnote{Taken from \url{https://cs.nyu.edu/~roweis/data.html}}
Each set is transformed to only have values in between 0 and 1. As training/test split, we have
\begin{itemize}
	\item synthetic: 6700 / 3300
	\item mnist: 60000 / 10000
	\item frey: 1316 / 649
\end{itemize}

The synthetic data is constructed in the fashion of \citet{Lee2010}. 
For $k=2$, $N = 10000 $ and $d = 200$, we generate two matrices $A \in \R^{N\times k}$ and $B \in \R^{d \times k}$.  $A$ is identifiable with principal components and B with (sparse) loading vectors of a PCA. The two-dimensional principal components $a^{(i)} (i=1,\ldots,N)$ of $A$ are drawn from normal distributions, so that $a^{(i)}_1 \sim \Nor(0,0.09)$ and  $a^{(i)}_2 \sim \Nor(0,0.25)$. The sparse loading vectors are constructed by setting $B$ to zero except for $b_{j,1} = 1, j =1,\ldots,20$ and $b_{j,2} = 1, j =21,\ldots,40$.

Given $A$ and $B$ we calculate
\[
\bXi := A \cdot B^T
\]
and the probability matrix $\Pi$, with
\[
\Pi = \sigma( \bXi),
\]
where we apply the sigmoid function $\sigma(\cdot)$  element-wise.
We then use the probabilities $\Pi^{(i)}_{j}$ to independently draw samples 
\[
x^{(i)}_{j} \sim Bern(\Pi^{(i)}_{j}).
\]
With the data $X_{Data} := (x^{(i)}_{j})_{i=1,\ldots,N;j=1,\ldots,d}$, we conduct the simulation.

\subsection{Initialization}\label{subsubsec:Initialization}

In each simulation, we compare two competing VAE with the same architecture and different initialization. We call these initializations ``Bench'' (benchmark) and ``MLE-B'' (MLE - based).

\begin{itemize}
\item For the benchmark, all network weights and biases are initialized as proposed by \citet{He2015}. This initialization particularly considers rectifier non-linearities. 
\item For MLE-B we use the same initialization as for the benchmark except for the weights and biases of the ``$\bmu_z$''-, ``$\log\bsigma_z^2$''- and ``$\hat{x}$''-layer. \added{The MLE-B initialization is cheap to calculate as only a singular value decomposition of the training data is needed for the maximum likelihood estimates. }

For the ``$\hat{x}$''-layer we use $\hat{W}$ as weights and $\hat{b}$ as bias. In case of over-parametrized nets, more edges lead into the affected layers than we need for MLE-B. This problem only concerns the weights and not the biases. We solve this by initializing not needed dimensions of the weights with zero. 

For the ``$\bmu_z$''-layer we use $W_e := \dfrac{1}{\hat{\varphi} \beta} \bhSigma_z \hat{W}^T$ as weights and $b_e := \dfrac{1}{\hat{\varphi} \beta} \bhSigma_z \hat{W}^T \bar{x}$ as bias. 

For the ``$\log\bsigma_z^2$''-layer we use the log of a diagonal \eqref{eq:optSigma with opt W} as bias. The weights a  initialized according to \citet{He2015}.

\end{itemize}

\newpage

\subsection{\added{Latent space activity: Histogram Distance}}

\added{To compare the resulting Activity statistics after training, we calculate a distance between two histograms. If we choose the amount of bins to be $b\in\N$, the histograms can be represented as $b$-dimensional, integer valued vectors $x,y \in \N^{b}$, with $\sum_{i} x_i = \sum_{i}y_i = \kappa$ and $x_i,y_i \geq 0$ for all $i$, where $\kappa$ denotes the latent dimension. We define the distance between two histograms as }

\begin{equation}\label{eq:histogramm distance}
\added{d(x,y) := 1 - \dfrac{\sum_{i=1}^{b} \min\left\{x_i,y_i\right\}}{\kappa}.}
\end{equation}

\added{This distance fulfils all properties of a metric on the defined set of histograms. Further, if two histogram representations are equal, we have the minimal value $d(x,y) =0$. The maximum distance of $d(x,y)=1$ is achieved, if bars never overlap.

For our experiments, to calculate the distance we choose 10 bins, given by

\[
[0,0.1),
[0.1,0.2)
, \ldots, 
[0.9,\infty).
\]
}

\subsection{Latent space activity: Bernoulli Case}\label{app:Bernoulli Activities}

\begin{figure*}[!htb]
	\begin{center}
		\centerline{\includegraphics[width=\linewidth]{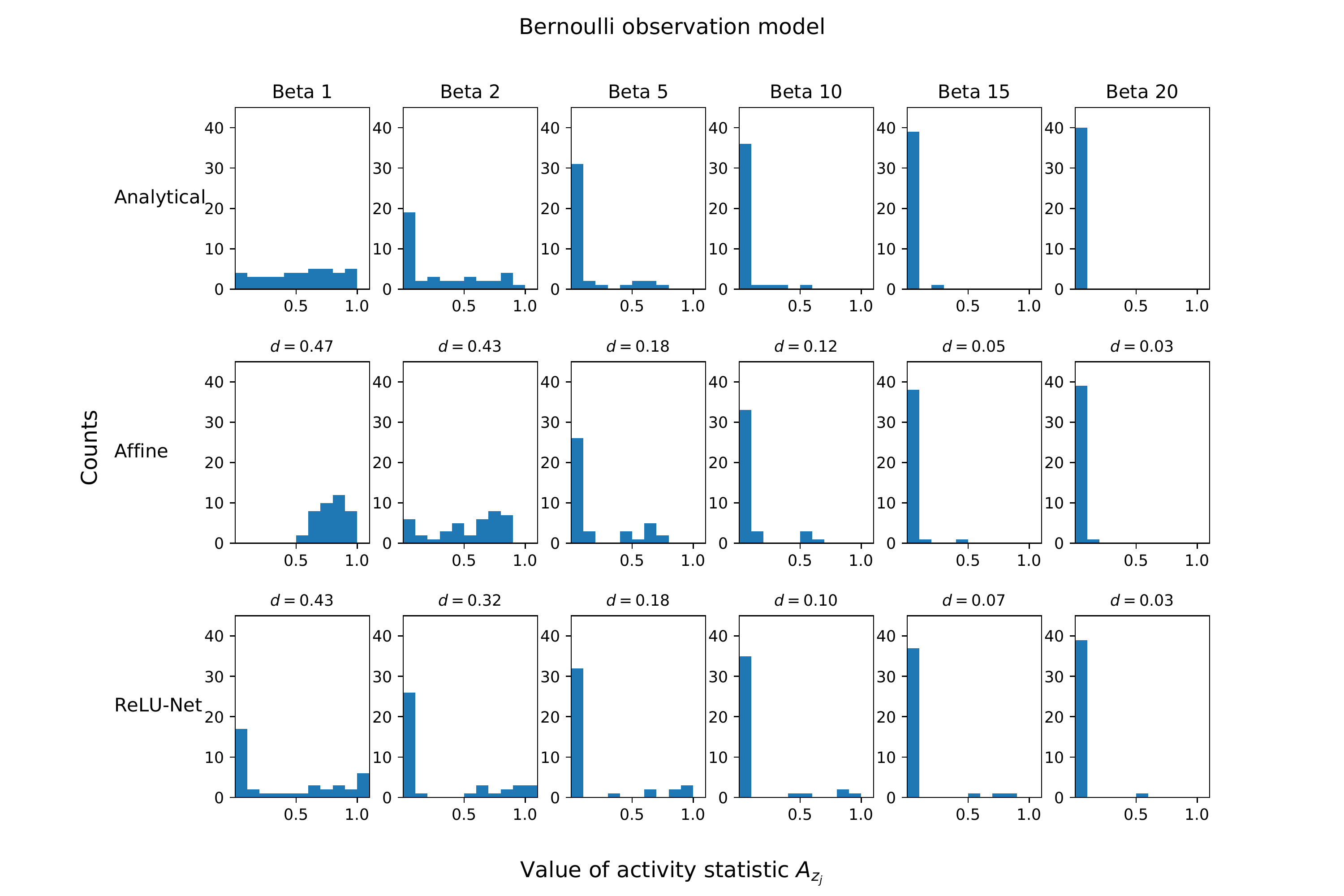}}
		\caption{The figure shows\added{ histograms of} the activities for 40 latent dimensions for our analytical calculation and an Affine/ ReLU-Net decoder (as described in appendix \ref{app:Architecture}) after training. \replaced{We have considered}{Considered is} the mnist data set with a Bernoulli observation model. \added{Above each Affine and ReLU-Net histogram plot, we show the distance ($\in [0,1]$, lower is better) as defined in \eqref{eq:histogramm distance} to the analytical histogram. }}
		\label{fig:Activities_41_46}
	\end{center}
	\vskip -0.35in
\end{figure*}

{\renewcommand{\arraystretch}{0.5}
	\begin{table}[!htb]
		\caption{\added{We display the distances ($\in [0,1]$, lower is better) of the histograms to the corresponding analytical calculation for the Bernoulli observation model. Displayed are the results of 10 simulations as ``mean$\pm$std''.}}
		\label{tab:Bernoulli Histogram distance}
		\centering
		\begin{tabular}{ccccccc}
			\toprule
			& Beta 1 & Beta 2 &Beta 5 &Beta 10 &Beta 15 & Beta 20 \\
			\midrule
		Affine & 0.48 ($\pm$0.03) & 0.38 ($\pm$0.03) & 0.15 ($\pm$0.03) & 0.08 ($\pm$0.01) & 0.04 ($\pm$0.02) & 0 ($\pm$0.01) \\ 
		ReLU & 0.52 ($\pm$0.03) & 0.35 ($\pm$0.03) & 0.16 ($\pm$0.02) & 0.09 ($\pm$0.01) & 0.07 ($\pm$0) & 0.03 ($\pm$0) \\ 
		\bottomrule
		\end{tabular}
\end{table}}

\newpage
\subsection{Simulation results}\label{app:Simulation results}

\begin{figure*}[ht]
	\vskip -0.2in
	\begin{center}
		\centerline{\includegraphics[width=1.2\linewidth]{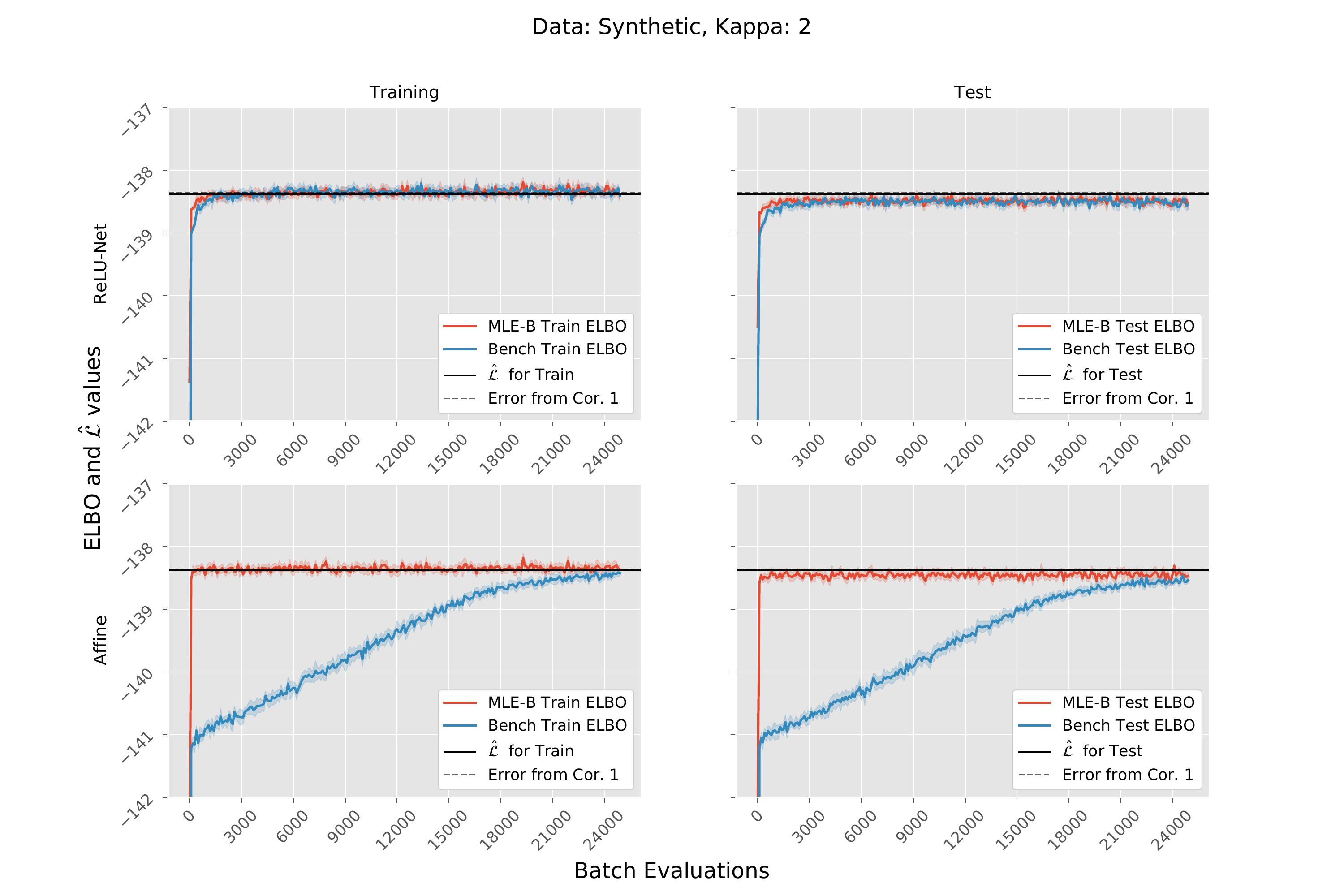}}
		\caption{The pictures show the setups ReLU-Net and Affine with synthetic data, $\kappa=2$. Displayed are the ELBOs of both initialisations MLE-B and Bench as  well as the lower bound $\hELBO$ and the expected error \added{$\E_{q_{\hat{\phi}}}[R_2(\hat{\theta})]$ }as provided by Corollary \ref{corr: F error}, calculated based on MLE. On the left the \added{ELBO} values are calculated with the training data and on the right with test data. \added{The curves are based on simulations with 10 different seeds. We display the average training performances with pointwise 0.95 confidence intervals.} }
		\label{fig:data_synthetic_kappa_2_act_sigmoid_plot}
	\end{center}
	\vskip -0.8in
\end{figure*}

\begin{figure*}[ht]
	\vskip -0.2in
	\begin{center}
		\centerline{\includegraphics[width=1.2\linewidth]{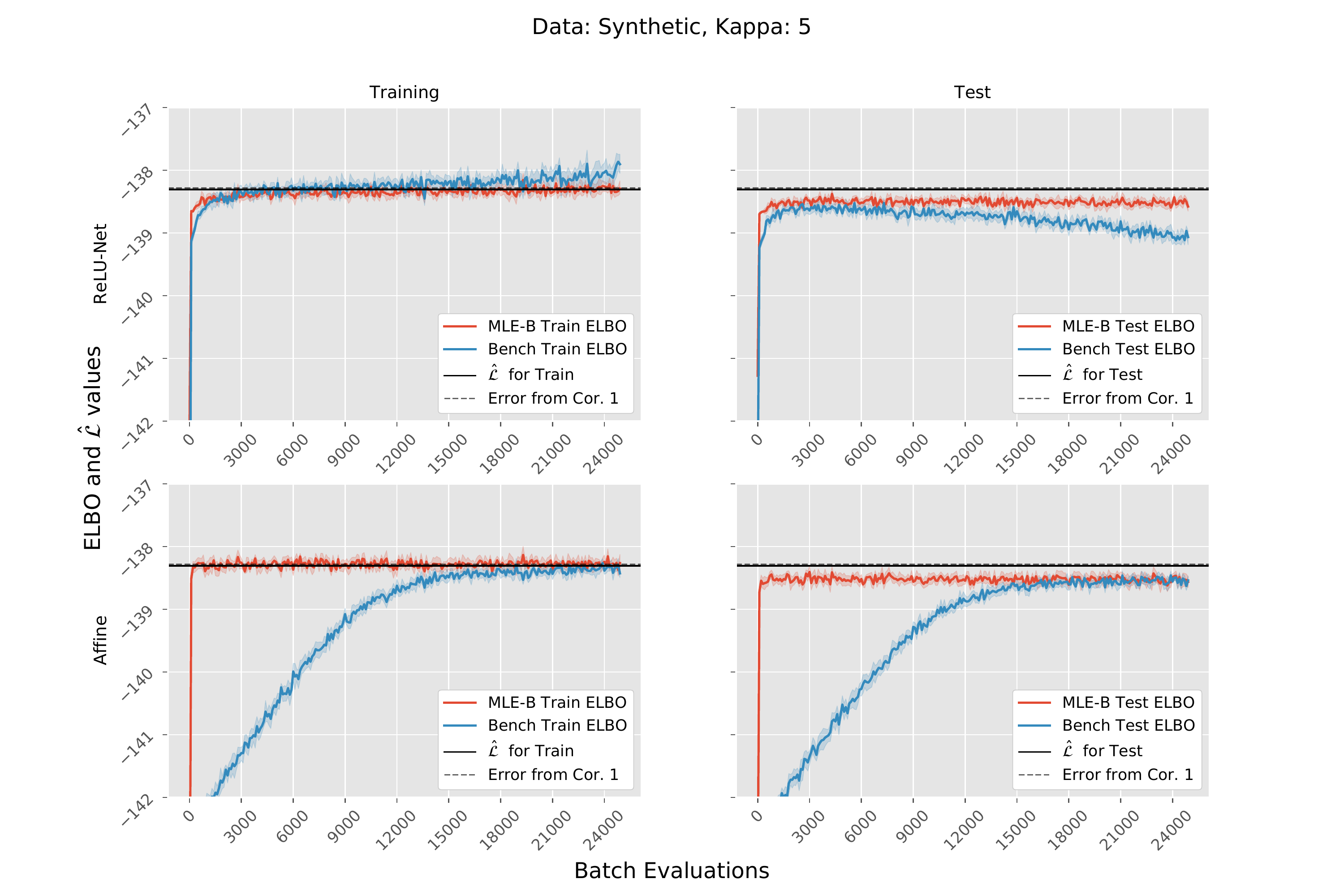}}
		\caption{The pictures show the setups ReLU-Net and Affine with synthetic data, $\kappa=5$. Displayed are the ELBOs of both initialisations MLE-B and Bench as  well as the lower bound $\hELBO$ and the expected error \added{$\E_{q_{\hat{\phi}}}[R_2(\hat{\theta})]$ }as provided by Corollary \ref{corr: F error}, calculated based on MLE. On the left the \added{ELBO} values are calculated with the training data and on the right with test data. \added{The curves are based on simulations with 10 different seeds. We display the average training performances with pointwise 0.95 confidence intervals.} }
		\label{fig:data_synthetic_kappa_5_act_sigmoid_plot}
	\end{center}
	\vskip -0.8in
\end{figure*}

\begin{figure*}[ht]
	\vskip -0.2in
	\begin{center}
		\centerline{\includegraphics[width=1.2\linewidth]{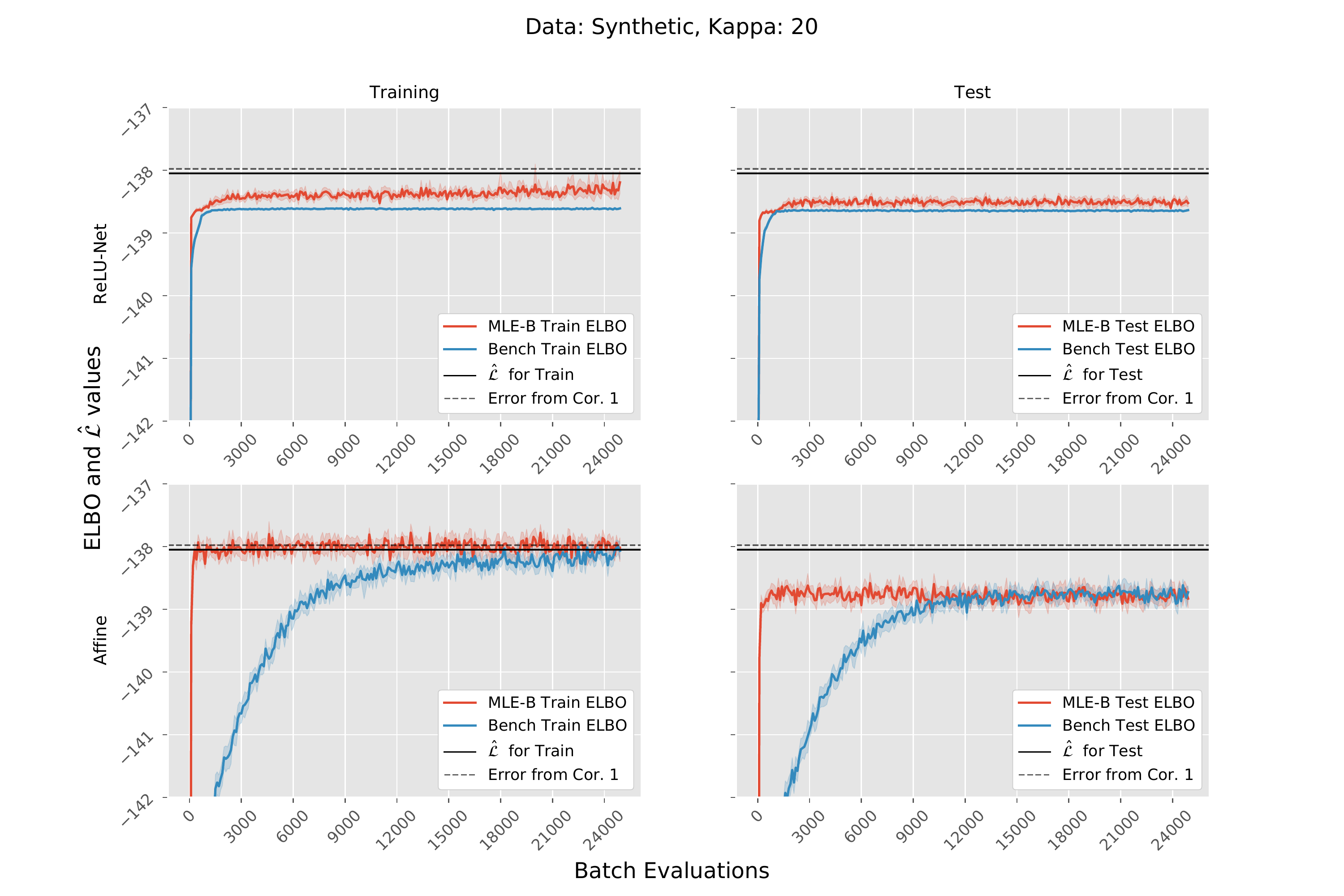}}
		\caption{The pictures show the setups ReLU-Net and Affine with synthetic data, $\kappa=20$. Displayed are the ELBOs of both initialisations MLE-B and Bench as  well as the lower bound $\hELBO$ and the expected error \added{$\E_{q_{\hat{\phi}}}[R_2(\hat{\theta})]$ }as provided by Corollary \ref{corr: F error}, calculated based on MLE. On the left the \added{ELBO} values are calculated with the training data and on the right with test data. \added{The curves are based on simulations with 10 different seeds. We display the average training performances with pointwise 0.95 confidence intervals.} }
		\label{fig:data_synthetic_kappa_20_act_sigmoid_plot}
	\end{center}
	\vskip -0.8in
\end{figure*}

\begin{figure*}[ht]
	\vskip -0.2in
	\begin{center}
		\centerline{\includegraphics[width=1.2\linewidth]{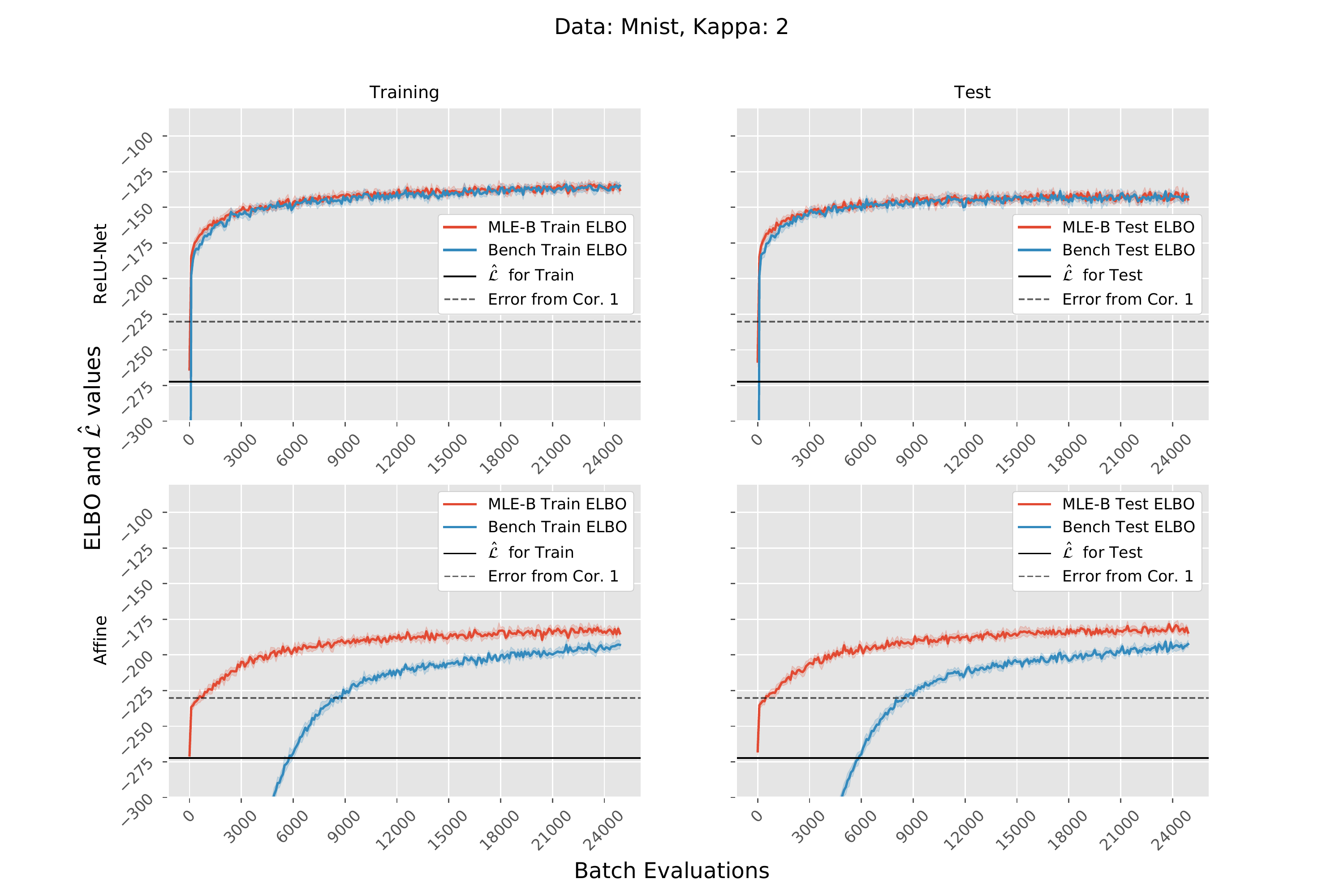}}
		\caption{The pictures show the setups ReLU-Net and Affine with mnist data, $\kappa=2$. Displayed are the ELBOs of both initialisations MLE-B and Bench as  well as the lower bound $\hELBO$ and the expected error \added{$\E_{q_{\hat{\phi}}}[R_2(\hat{\theta})]$ }as provided by Corollary \ref{corr: F error}, calculated based on MLE. On the left the \added{ELBO} values are calculated with the training data and on the right with test data. \added{The curves are based on simulations with 10 different seeds. We display the average training performances with pointwise 0.95 confidence intervals.} }
		\label{fig:data_mnist_kappa_2_act_sigmoid_plot}
	\end{center}
	\vskip -0.8in
\end{figure*}

\begin{figure*}[ht]
	\vskip -0.2in
	\begin{center}
		\centerline{\includegraphics[width=1.2\linewidth]{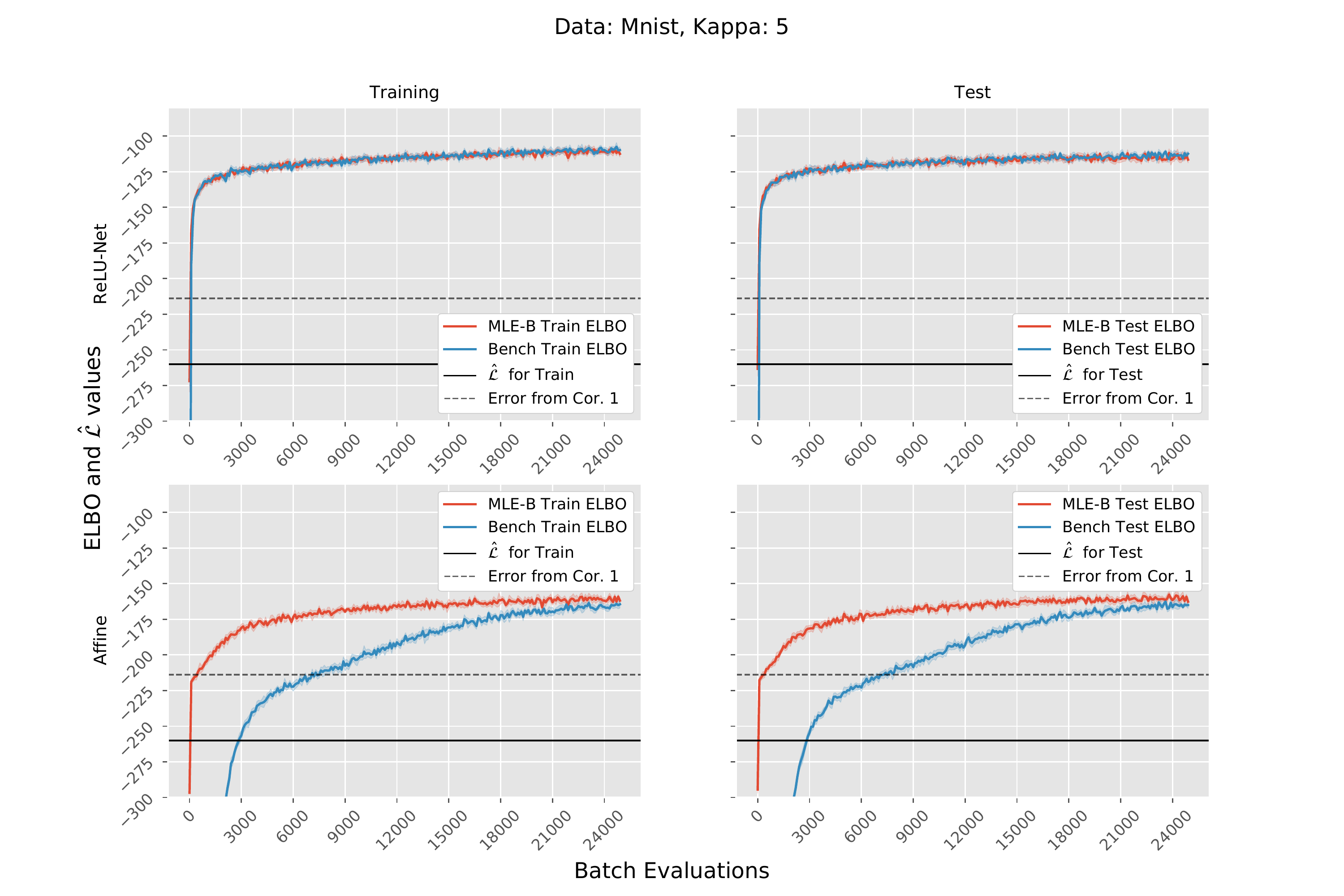}}
		\caption{The pictures show the setups ReLU-Net and Affine with mnist data, $\kappa=5$. Displayed are the ELBOs of both initialisations MLE-B and Bench as  well as the lower bound $\hELBO$ and the expected error \added{$\E_{q_{\hat{\phi}}}[R_2(\hat{\theta})]$ }as provided by Corollary \ref{corr: F error}, calculated based on MLE. On the left the \added{ELBO} values are calculated with the training data and on the right with test data. \added{The curves are based on simulations with 10 different seeds. We display the average training performances with pointwise 0.95 confidence intervals.} }
		\label{fig:data_mnist_kappa_5_act_sigmoid_plot}
	\end{center}
	\vskip -0.8in
\end{figure*}

\begin{figure*}[ht]
	\vskip -0.2in
	\begin{center}
		\centerline{\includegraphics[width=1.2\linewidth]{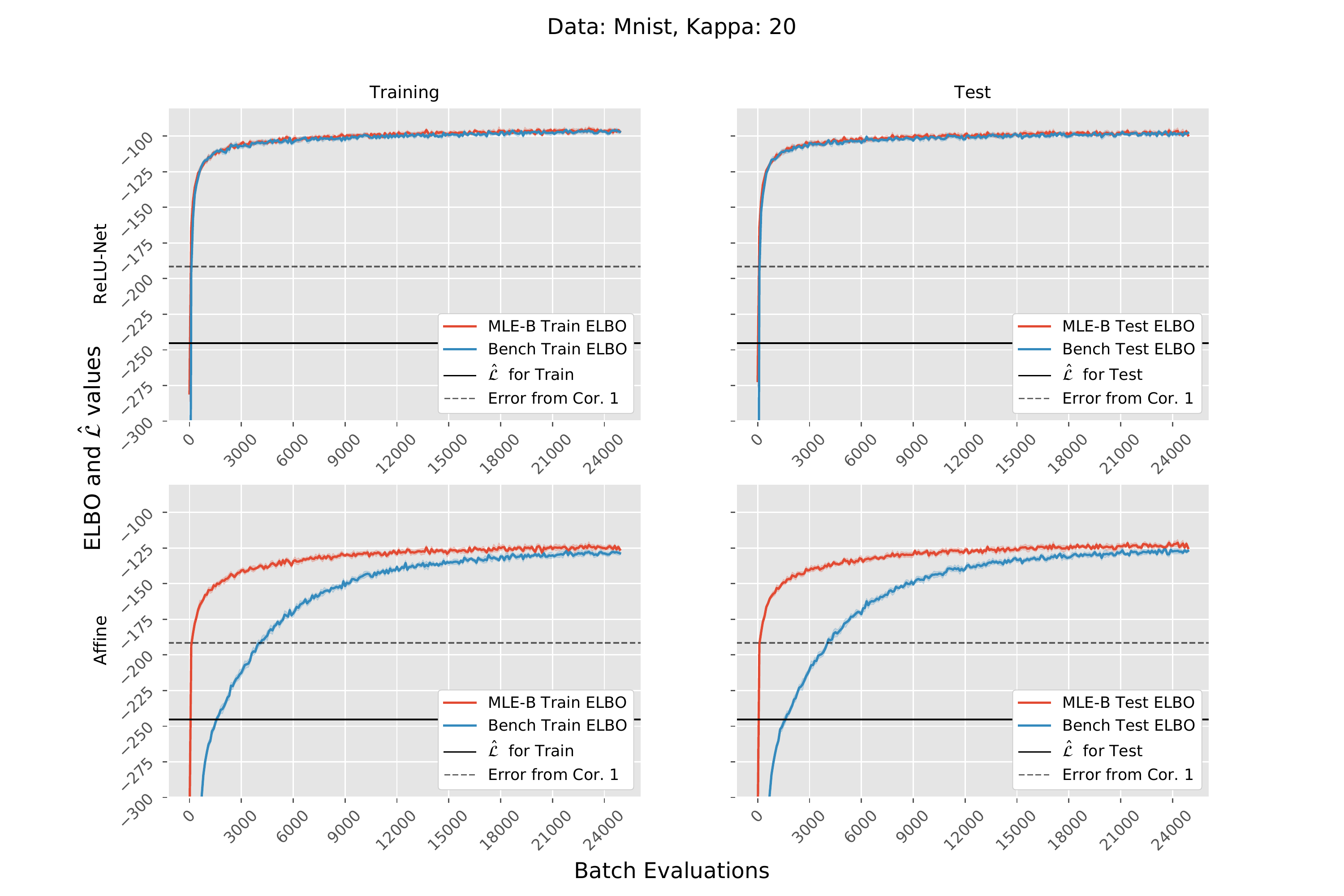}}
		\caption{The pictures show the setups ReLU-Net and Affine with mnist data, $\kappa=20$. Displayed are the ELBOs of both initialisations MLE-B and Bench as  well as the lower bound $\hELBO$ and the expected error \added{$\E_{q_{\hat{\phi}}}[R_2(\hat{\theta})]$ }as provided by Corollary \ref{corr: F error}, calculated based on MLE. On the left the \added{ELBO} values are calculated with the training data and on the right with test data. \added{The curves are based on simulations with 10 different seeds. We display the average training performances with pointwise 0.95 confidence intervals.} }
		\label{fig:data_mnist_kappa_20_act_sigmoid_plot}
	\end{center}
	\vskip -0.8in
\end{figure*}

\begin{figure*}[ht]
	\vskip -0.2in
	\begin{center}
		\centerline{\includegraphics[width=1.2\linewidth]{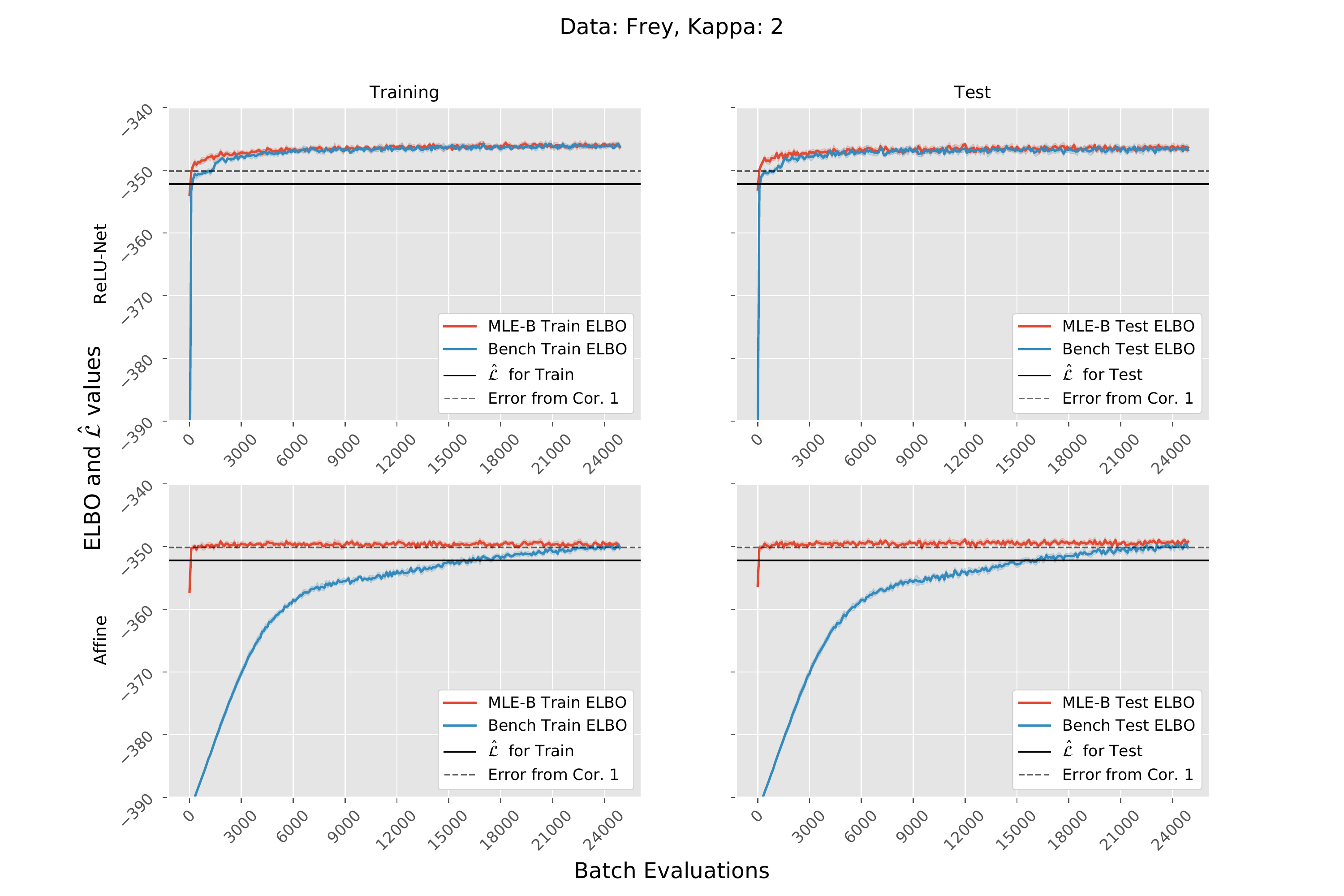}}
		\caption{The pictures show the setups ReLU-Net and Affine with frey data, $\kappa=2$. Displayed are the ELBOs of both initialisations MLE-B and Bench as  well as the lower bound $\hELBO$ and the expected error \added{$\E_{q_{\hat{\phi}}}[R_2(\hat{\theta})]$ }as provided by Corollary \ref{corr: F error}, calculated based on MLE. On the left the \added{ELBO} values are calculated with the training data and on the right with test data. \added{The curves are based on simulations with 10 different seeds. We display the average training performances with pointwise 0.95 confidence intervals.} }
	\end{center}
	\vskip -0.8in
\end{figure*}

\begin{figure*}[ht]
	\vskip -0.2in
	\begin{center}
		\centerline{\includegraphics[width=1.2\linewidth]{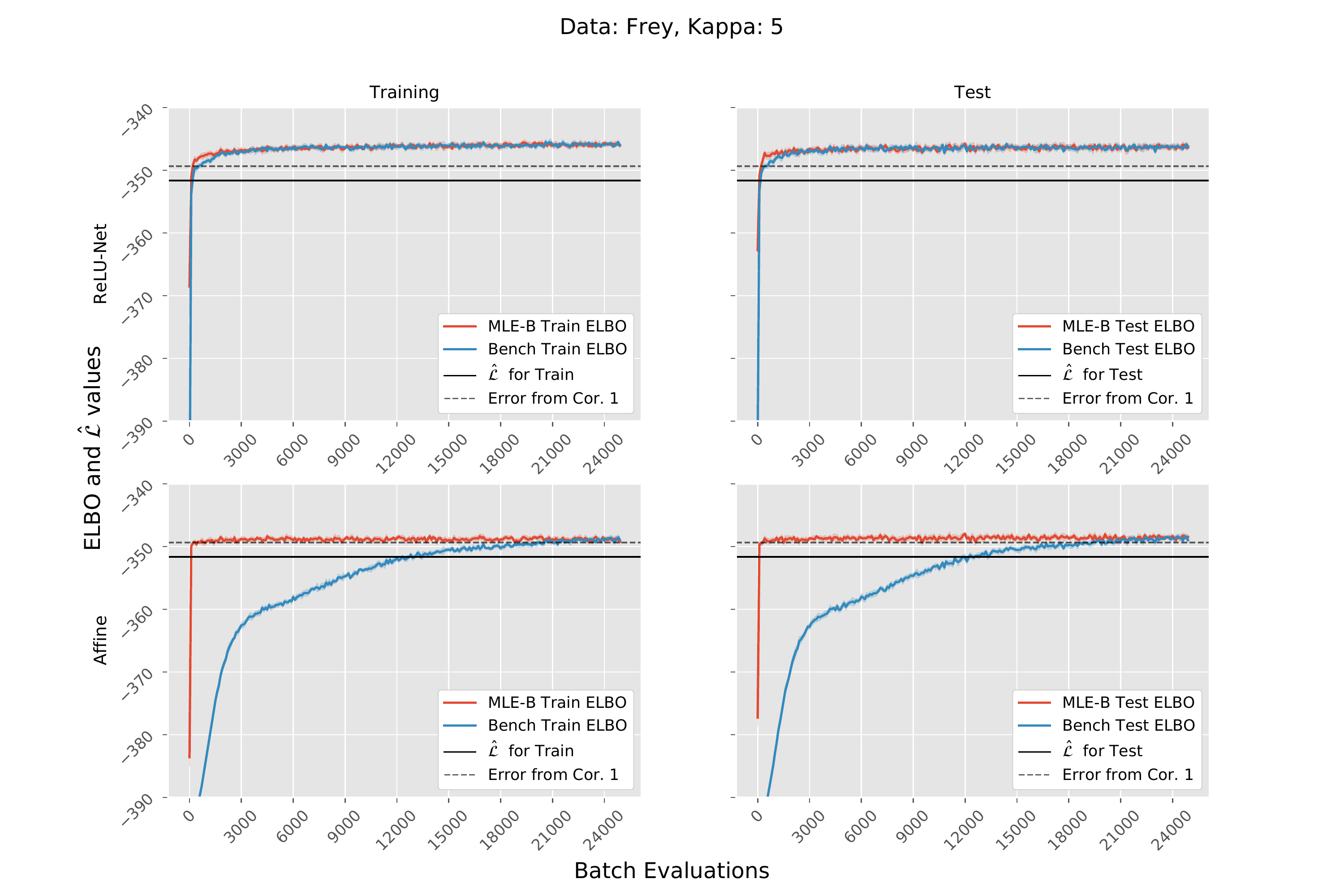}}
		\caption{The pictures show the setups ReLU-Net and Affine with frey data, $\kappa=5$. Displayed are the ELBOs of both initialisations MLE-B and Bench as  well as the lower bound $\hELBO$ and the expected error \added{$\E_{q_{\hat{\phi}}}[R_2(\hat{\theta})]$ }as provided by Corollary \ref{corr: F error}, calculated based on MLE. On the left the \added{ELBO} values are calculated with the training data and on the right with test data. \added{The curves are based on simulations with 10 different seeds. We display the average training performances with pointwise 0.95 confidence intervals.} }
		\label{fig:data_frey_kappa_5_act_sigmoid_plot}
	\end{center}
	\vskip -0.8in
\end{figure*}

\begin{figure*}[ht]
	\vskip -0.2in
	\begin{center}
		\centerline{\includegraphics[width=1.2\linewidth]{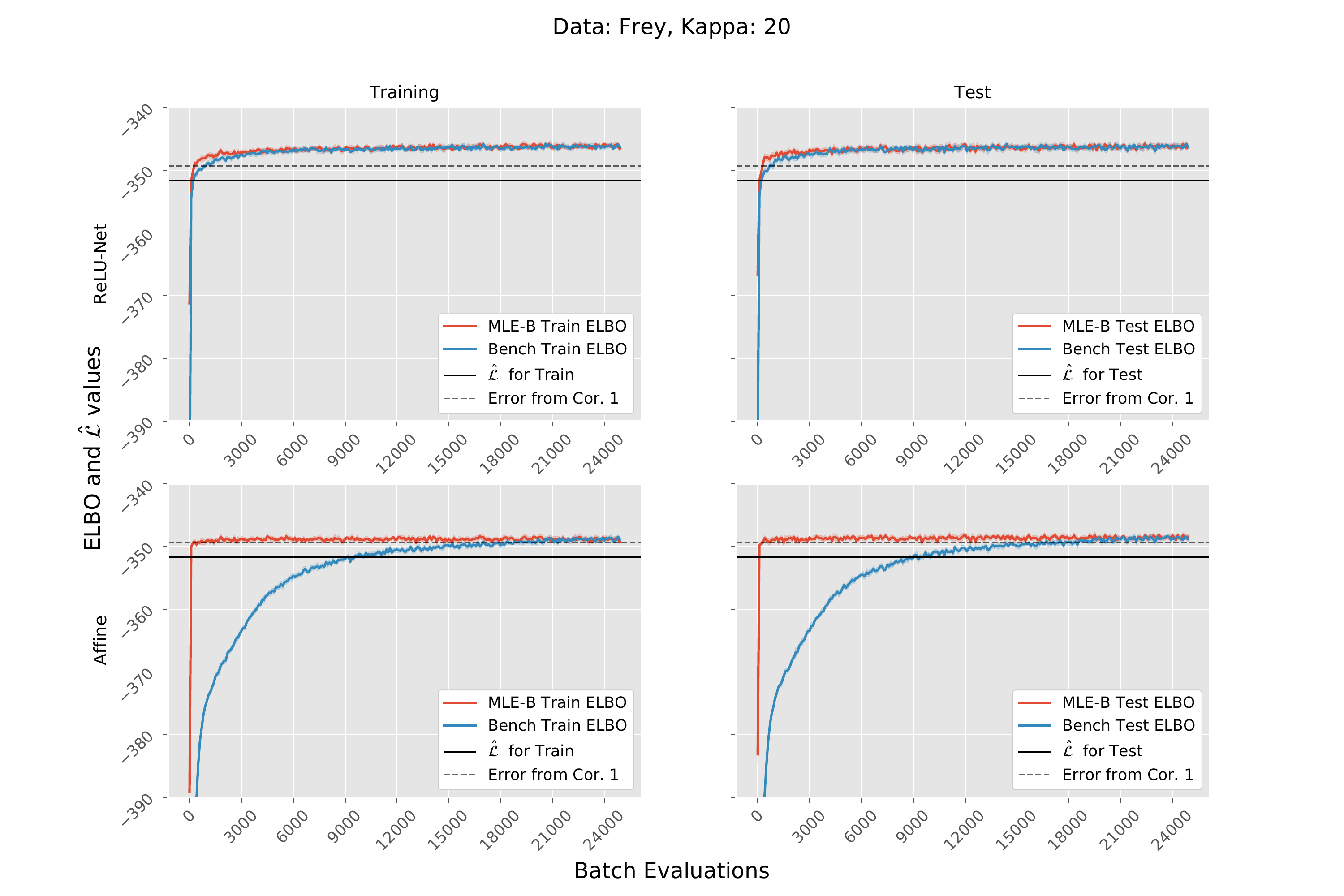}}
		\caption{The pictures show the setups ReLU-Net and Affine with frey data, $\kappa=20$. Displayed are the ELBOs of both initialisations MLE-B and Bench as  well as the lower bound $\hELBO$ and the expected error \added{$\E_{q_{\hat{\phi}}}[R_2(\hat{\theta})]$ }as provided by Corollary \ref{corr: F error}, calculated based on MLE. On the left the \added{ELBO} values are calculated with the training data and on the right with test data. \added{The curves are based on simulations with 10 different seeds. We display the average training performances with pointwise 0.95 confidence intervals.} }
		\label{fig:data_frey_kappa_20_act_sigmoid_plot}
	\end{center}
	\vskip -0.8in
\end{figure*}

\clearpage
\newpage

\small

\bibliographystyle{humannat}
\bibliography{VAEExpFam.bib}

\normalsize

\end{document}